\def\withsuppl{1}
\def\arxivversion{1}
\def\arxivfigfile{arxiv_mix/arxiv_mix_relabelled_D09_side_dojo_first}

\documentclass[10pt,twocolumn,letterpaper]{article}

\usepackage[pagenumbers]{cvpr} 

\usepackage{amsmath,amssymb,amsthm}
\usepackage{booktabs}
\usepackage{array}
\usepackage{tabularx}
\usepackage{longtable}
\usepackage{multirow}
\usepackage{enumitem}
\usepackage{xspace}
\usepackage{pifont}
\usepackage{afterpage}
\usepackage{cuted}
\usepackage{colortbl}
\usepackage{float}  

\theoremstyle{definition}

\theoremstyle{plain}
\newtheorem{proposition}{Proposition}

\newcommand{\ours}{RoboHarness\xspace}

\usepackage[skins]{tcolorbox}  
\definecolor{boxblue}{RGB}{238,242,252}
\definecolor{boxframe}{RGB}{206,214,236}
\newcommand{\abstracttext}{
The rise of large language models (LLMs) has opened a path toward generalist embodied AI, raising the question of how LLMs can act in the physical world. We argue that coding agents can directly control robots and generalize across diverse embodiments. All they need is a simple yet effective robot harness. We present RoboHarness: a minimal design paradigm for robot harnesses that enables robot agents to generalize across hardware, tasks, scenes, and toolsets. With minimal coding harness, trajectories from different embodiments and toolsets can be pooled to scale up training, and the harness can keep evolving alongside the model, showing the potential for multi-task and open-ended scaling. The implementation is open source.
}
\usepackage{fontawesome5}  

\definecolor{cvprblue}{rgb}{0.21,0.49,0.74}
\usepackage[pagebackref,breaklinks,colorlinks,allcolors=cvprblue]{hyperref}

\def\confName{CVPR}
\def\confYear{2027}

\title{\resizebox{\textwidth}{!}{\fontsize{22}{27}\selectfont\bfseries
\texorpdfstring{\raisebox{-0.3\height}{\includegraphics[height=1.25em]{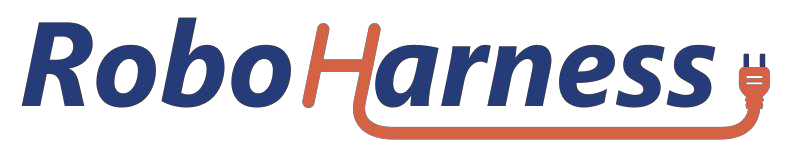}}}{RoboHarness}: Minimal Coding Harness, General Robot Agents}}

\author{%
{\large Bince Qu\textsuperscript{1}\quad Xiyan Jiang\textsuperscript{1}\quad Zheng Liu\textsuperscript{3,$\dagger$}\quad Wei Chen\textsuperscript{1}\quad Bo Zhang\textsuperscript{1,$\dagger$}}\\[2pt]
{\small\textsuperscript{1}State Key Lab of CAD\&CG, Zhejiang University\quad \textsuperscript{3}The Hong Kong Polytechnic University\quad \textsuperscript{$\dagger$}Co-corresponding authors}\\[1pt]
{\small \href{https://roboharness-lab.github.io/}{\raisebox{-0.25em}{\includegraphics[height=1.15em]{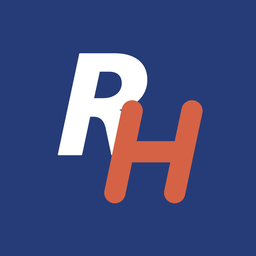}}}~\url{https://roboharness-lab.github.io/}}\\[1pt]
{\small \href{https://github.com/BinceQu/RoboHarness}{\faGithub}~\url{https://github.com/BinceQu/RoboHarness}}
}

\begin{document}
\onecolumn
\begin{tcolorbox}[enhanced,colback=boxblue,colframe=boxframe,boxrule=0.5pt,arc=7pt,left=14pt,right=14pt,top=6pt,bottom=6pt,width=0.88\textwidth,center,halign=flush left]
{\raggedright\fontsize{23}{30}\selectfont\bfseries
\raisebox{-0.3\height}{\includegraphics[height=1.2em]{figures/logo/roboharness_logo.png}}: Minimal Coding Harness,\\
General Robot Agents\par}
\vspace{6pt}
{\raggedright\bfseries\fontsize{12.5}{16}\selectfont Bince Qu\textsuperscript{1}\quad Xiyan Jiang\textsuperscript{1}\quad Zheng Liu\textsuperscript{3,$\dagger$}\quad Wei Chen\textsuperscript{1}\quad Bo Zhang\textsuperscript{1,$\dagger$}\par}
\vspace{2pt}
{\raggedright\fontsize{10}{13.5}\selectfont \textsuperscript{1}State Key Lab of CAD\&CG, Zhejiang University\\
\textsuperscript{3}The Hong Kong Polytechnic University\quad \textsuperscript{$\dagger$}Co-corresponding\par}
\vspace{1.5pt}
{\raggedright\small\ttfamily bincequ@zju.edu.cn, zhengliu1026@gmail.com,\\
chenvis@zju.edu.cn, bo.zhang@zju.edu.cn\par}
\vspace{3.5pt}
{\raggedright\small \href{https://roboharness-lab.github.io/}{\raisebox{-0.25em}{\includegraphics[height=1.15em]{figures/logo/roboharness_icon.png}}}~\url{https://roboharness-lab.github.io/}\\[1.5pt]
\href{https://github.com/BinceQu/RoboHarness}{\faGithub}~\url{https://github.com/BinceQu/RoboHarness}\par}
\vspace{5pt}
{\fontsize{10}{13.8}\selectfont\upshape\abstracttext\par}
\end{tcolorbox}
\vspace{6pt}
\begin{center}
\begin{minipage}{0.88\textwidth}
\includegraphics[width=\linewidth]{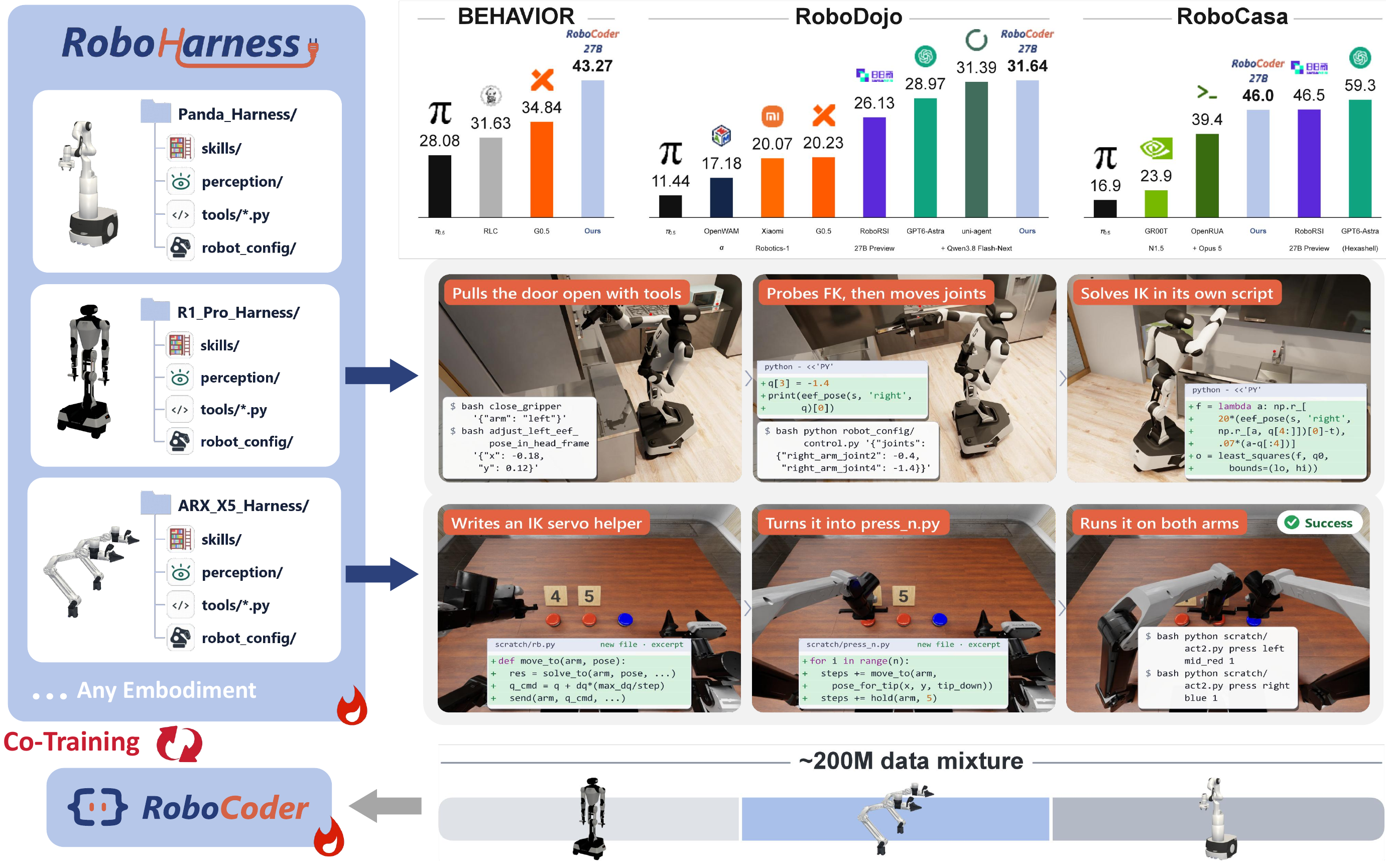}\par
\vspace{-4pt}
\captionsetup{hypcap=false}
\captionof{figure}{\textbf{\ours is a minimal robot harness design contract that enables a wide range of embodiments to be trained in one agent.} \textbf{Left:} each robot is a workspace with \texttt{skills/}, \texttt{perception/}, \texttt{tools/} and \texttt{robot\_config/}; trajectories from all embodiments train one model (RoboCoder), co-trained with its harness. \textbf{Top:} scores of RoboCoder-27B under \ours and reference systems on BEHAVIOR, RoboDojo and RoboCasa365. \textbf{Middle:} rollouts on R1~Pro and ARX~X5 where the agent writes and runs its own IK scripts. \textbf{Bottom:} the $\sim$200M-token data mixture.}
\captionsetup{hypcap=true}
\label{fig:overview}
\end{minipage}
\end{center}
\twocolumn
\addtocontents{toc}{\protect\setcounter{tocdepth}{-1}}
\section{Introduction}
\label{sec:introduction}

General-purpose large language models (LLMs) already possess broad robotics knowledge and spatial understanding, showing strong potential for generalist embodied AI \citep{chu2026openrua,yuan2026unirobotagent,huang2023voxposer}. Building on this, researchers have designed specialized harnesses for robot tasks, boosting agent performance across manipulation benchmarks \citep{gao2026hexashell,wang2026auriga}. However, current robot harnesses share two common problems:
\begin{itemize}[leftmargin=*,nosep]
    \item \textbf{Human-defined action tools constrain the learning space.} Fixed, opaque APIs hide underlying code and restrict agents to pre-made primitives \citep{liu2026guava}, capping both reachable motions and achievable return.
    \item \textbf{Embodiment-specific harnesses fragment data and block scaling.} Incompatible per-robot APIs and primitives create disjoint action spaces across embodiments \citep{gao2026hexashell,si2026pyrualean,gong2026tinyastral,pear2026,huang2026rax}, preventing trajectory pooling and cross-embodiment transfer.
\end{itemize}

This mirrors Sutton's \emph{Bitter Lesson} \citep{sutton2019bitter}: building human domain knowledge into a system helps in the short term, but ultimately plateaus and constrains what the agent can learn. In the long run, the methods that scale best are those that avoid hard-coded priors and instead provide a general, expressive space where search and learning can discover complex behaviors on their own.

We therefore argue that only a \emph{minimal robot harness} can unlock true generalization across embodiments, scenes, and tasks. Instead of locking the agent behind fixed APIs, the harness should stay minimal and give the agent full freedom to read hardware docs, inspect sensor streams, and write its own perceptors and controllers.

We present \textbf{\ours} (Fig.~\ref{fig:overview}), a minimal design paradigm for robot harnesses that enables robot agents to generalize across hardware, tasks, scenes, and toolsets. Specifically, our contributions are:
\begin{itemize}[leftmargin=*,nosep]
    \item \textbf{A unified minimal abstraction for Robot Coding Agents (RCAs).} Distinct from fixed-tool robot-use agents (RUAs), we formulate every robot task as a coding project ($\text{Task} = \text{task} + \text{harness}$, Fig.~\ref{fig:task-def}), exposing robot config, sensors streams, and editable tools in one workspace. Theoretically, minimal harness maximize reachable motions and return ceilings (Props.~\ref{prop:nested}--\ref{prop:edit}, Sec.~\ref{sec:supp-formulation}); empirically, porting third-party harnesses into \ours preserves or improves performance (Tab.~\ref{tab:port}).
    \item \textbf{Cross-embodiment and cross-harness training gains.} Unifying robots and toolsets under one contract enables conflict-free data pooling and cross-domain transfer (Props.~\ref{prop:mix}--\ref{prop:aux}, Sec.~\ref{sec:supp-formulation}). On our $\sim$200M-token dataset, pooled SFT outperforms single-domain training on RoboDojo and RoboCasa365 (Fig.~\ref{fig:train-ablation}c--d), cuts validation loss by $8.8\%$ on an unseen toolset and $8.9\%$ on an unseen robot (Fig.~\ref{fig:train-ablation}e--f), and enables continual harness--model co-training without data invalidation (Figs.~\ref{fig:cotrain}--\ref{fig:cotrain-curve}).
\end{itemize}

\providecommand{\ckY}{{\color[rgb]{0.08,0.50,0.22}\ding{51}}}
\providecommand{\ckD}{{\color[rgb]{0.08,0.50,0.22}\ding{51}\textsuperscript{\scriptsize \textdagger}}}
\providecommand{\ckN}{{\color[gray]{0.68}\ding{55}}}

\begin{table*}[!t]
\caption{\textbf{Comparison of robot harnesses.} \ckY\ supported; \ckN\ not supported; \ckD\ supported with conditions or caveats. Runtime entries follow released code; RoboRSI reports joint training but does not disclose its update schedule. Criteria and caveats are in Sec.~\ref{sec:supp-audit}. The training entries for Ours follow Sec.~\ref{sec:training}.}
\label{tab:compare}
\centering\small
\setlength{\tabcolsep}{3.5pt}
\renewcommand{\arraystretch}{1.20}
\resizebox{\textwidth}{!}{%
\begin{tabular}{@{}l*{12}{c}@{}}
\hline
\raisebox{-.5\height}{\textbf{Harness for robot agent:}} & \raisebox{-.5\height}{\scriptsize OpenRUA~\citep{chu2026openrua}} & \raisebox{-.5\height}{\scriptsize RAX~\citep{huang2026rax}} & \raisebox{-.5\height}{\scriptsize HexaShell~\citep{gao2026hexashell}} & \raisebox{-.5\height}{\scriptsize Auriga~\citep{wang2026auriga}} & \raisebox{-.5\height}{\scriptsize uni-robot-agent~\citep{yuan2026unirobotagent}} & \raisebox{-.5\height}{\scriptsize PyRUA-Lean~\citep{si2026pyrualean}} & \raisebox{-.5\height}{\scriptsize ASPIRE~\citep{lu2026aspire}} & \raisebox{-.5\height}{\scriptsize TinyAstral~\citep{gong2026tinyastral}} & \raisebox{-.5\height}{\scriptsize RoboRSI~\citep{roborsi2026}} & \raisebox{-.5\height}{\scriptsize ReKep~\citep{huang2024rekep}} & \raisebox{-.5\height}{\scriptsize CoPa~\citep{huang2024copa}} & \raisebox{-.5\height}{\textbf{Ours}} \\
\hline
\noalign{\vskip 0.5pt}
\rowcolor{black!7}\textit{By design, the harness can:} & & & & & & & & & & & & \\
\hspace{8pt}Support any coding agent & \ckY & \ckY & \ckN & \ckN & \ckY & \ckY & \ckY & \ckN & \ckY & \ckN & \ckN & \ckY \\
\hspace{8pt}Support any robot embodiment & \ckY & \ckY & \ckY & \ckN & \ckY & \ckY & \ckY & \ckY & \ckY & \ckN & \ckN & \ckY \\
\hline
\rowcolor{black!7}\textit{At run time, the agent is able to:} & & & & & & & & & & & & \\
\hline
\hspace{8pt}Learn a new robot from its description & \ckY & \ckN & \ckN & \ckY & \ckY & \ckN & \ckD & \ckN & \ckN & \ckN & \ckN & \ckY \\
\hspace{8pt}Read, edit and create its own tools & \ckD & \ckN & \ckN & \ckN & \ckD & \ckN & \ckD & \ckN & \ckN & \ckN & \ckN & \ckY \\
\hspace{8pt}Offload sensor streams to save context & \ckY & \ckD & \ckY & \ckN & \ckD & \ckD & \ckD & \ckN & \ckN & \ckN & \ckN & \ckY \\
\hspace{8pt}Deploy on-demand sensor trackers & \ckY & \ckY & \ckY & \ckY & \ckY & \ckY & \ckD & \ckN & \ckN & \ckN & \ckN & \ckY \\
\hspace{8pt}Build fallback joint-level controllers & \ckY & \ckD & \ckN & \ckN & \ckN & \ckN & \ckY & \ckN & \ckD & \ckN & \ckN & \ckY \\
\hline
\rowcolor{black!7}\textit{In training, the framework can:} & & & & & & & & & & & & \\
\hspace{8pt}Support harness–model co-training & \ckN & \ckN & \ckN & \ckN & \ckN & \ckN & \ckN & \ckN & \ckD & \ckN & \ckN & \ckY \\
\hspace{8pt}Support cross-embodiment training & \ckN & \ckN & \ckN & \ckN & \ckN & \ckN & \ckN & \ckD & \ckN & \ckN & \ckN & \ckY \\
\hline
\end{tabular}}
\end{table*}
\providecommand{\ckY}{{\color[rgb]{0.08,0.50,0.22}\ding{51}}}
\providecommand{\ckD}{{\color[rgb]{0.08,0.50,0.22}\ding{51}\textsuperscript{\scriptsize \textdagger}}}
\providecommand{\ckN}{{\color[gray]{0.68}\ding{55}}}

\begin{figure}[H]
\centering
\includegraphics[width=\linewidth]{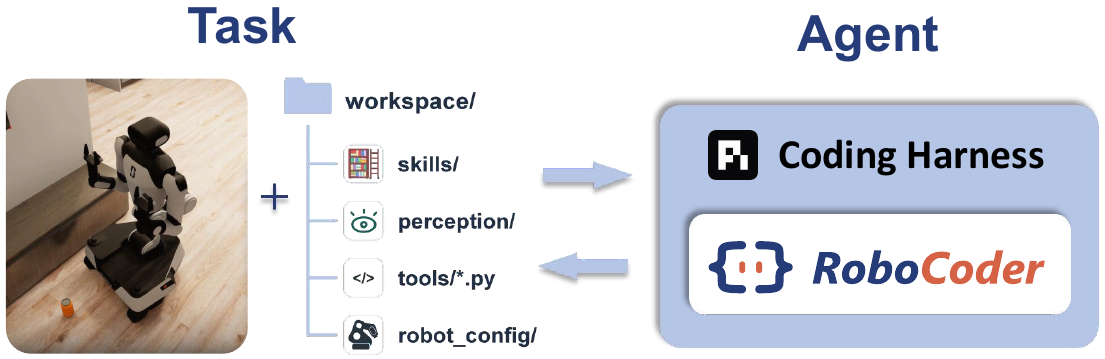}
\caption{\textbf{Robot Coding Agent: $\text{Task} = \text{task} + \text{harness}$. A robot is a project for a coding agent.} Each task pairs the physical robot task with an initial robot harness (as a workspace), operated directly by a coding agent.}
\label{fig:task-def}
\end{figure}

\section{Related Work}
\label{sec:related}

A robot harness connects a foundation model to perception, control, execution feedback and reusable tools.
We review the design of these harnesses, the improvement and training of robot agents, and model--harness co-optimization in general-purpose agents.
Tab.~\ref{tab:compare} compares the capabilities of the closest robot harnesses, using the released-code audit in Sec.~\ref{sec:supp-audit} of the supplementary material.

\subsection{Harnesses for Robot Agents}
\label{sec:related-agents}

\paragraph{From VLMs to robot-use agents.}
Robot-use agents (RUAs) use LLMs or VLMs to interpret observations, reason about tasks and issue robot actions.
SayCan selects available skills, Inner Monologue incorporates execution feedback, and VoxPoser grounds language goals in 3D value maps \citep{ahn2022saycan,huang2022innermonologue,huang2023voxposer}.
Thea and Guava organize interaction around callable perception and control tools \citep{wang2026thea,liu2026guava}.
PEAR illustrates native robot use: a reasoning model trained with simulation RL predicts structured Cartesian pose and gripper actions, which robot-specific controllers execute \citep{pear2026}.
These systems differ in their action interfaces; robot use does not require the model to generate executable programs.

\paragraph{Robot coding agents.}
Robot coding agents form a broader paradigm in which agents interact with robots primarily through executable code.
Code as Policies generates programs over perception and control APIs \citep{liang2023code}, while CaP-X studies how API abstraction and execution feedback affect coding policies \citep{fu2026capx}.
Agent-as-Policy incorporates program generation and revision into closed-loop robot control \citep{jia2026agp}.
OpenRUA and RAX run coding agents in workspaces where they build and execute robot programs \citep{chu2026openrua,huang2026rax}; HexaShell and PyRUA-Lean execute Python cells with persistent state \citep{gao2026hexashell,si2026pyrualean}.
This program interface lets the agent express branching, repeated sensing and reusable control logic, with its scope determined by the runtime and accessible robot APIs.

\paragraph{Cross-embodiment robot agents.}
Many recent agents report results on several robots, but these results depend on what the robots share.
Some agents reuse one toolset across bodies.
HexaShell runs on dual-arm ARX~X5, a fixed Panda, a mobile Panda and real dual-arm PiPER-X robots \citep{gao2026hexashell}.
Auriga carries one frozen harness built on end-effector waypoints from RoboDojo's ARX~X5 to RoboCasa365's Panda on a mobile base \citep{wang2026auriga}.
OpenRUA starts from the robot's ROS~2 configuration and a set of starter tools \citep{chu2026openrua}, and TinyAstral trains students on several benchmarks through one frozen harness \citep{gong2026tinyastral}.
Others depend on per-platform action primitives.

PyRUA-Lean trains a coder on one single-arm robot and tests it on dual-arm RoboTwin~2.0 and mobile RoboCasa365, but both targets use their own VLA policy as a primitive \citep{si2026pyrualean}.

These agents typically act through an interface, such as end-effector goals, a shared API or a platform-specific policy, that hides how each robot is built.
The results therefore show portability across configurations that the shared API can express, not generalization to arbitrary embodiments: a body the layer does not capture calls for new adapters, tools or primitives.
We instead fix only a joint-level command per robot, let the agent build the layers above it, and train one model on rollouts from Panda, dual-arm ARX~X5 and mobile bimanual R1~Pro (Sec.~\ref{sec:generalization}, \ref{sec:training}).

\subsection{Training Robot Agents}
\label{sec:related-training}

\paragraph{Self-improving robot agents.}
Robot agents can improve through experience by revising programs, skills and harnesses.
ASPIRE turns successful program repairs into reusable skill snippets \citep{lu2026aspire}, RHO searches multi-file policy repositories using execution feedback \citep{elmaaroufi2026rho}, and GaP refines graph-structured programs through simulated rehearsal \citep{chen2026gap}.
HarnessPAI evolves executable robot programs across rollouts and records failure-to-repair experience as skills; its converged programs also generate demonstrations for VLA/WAM fine-tuning, without feeding the updated models back into harness evolution \citep{wang2026harnesspai}.
SenseNova-RoboRSI also describes feedback-driven harness refinement with model weights fixed \citep{roborsi2026}.

\paragraph{Training models for robot agents.}
Another line of work trains the model that reasons and acts through the harness.
Guava distills 2,268 teacher trajectories into a compact 4B tool-use agent \citep{liu2026guava}.
PyRUA-Lean applies GRPO to a 4B coding model on LIBERO-PRO \citep{si2026pyrualean}, while AgenticRoboRL combines supervised fine-tuning with multi-turn reinforcement learning for a planner deployed through existing harnesses, keeping their execution components fixed \citep{liu2026agenticroborl}.
TinyAstral first evolves a harness with a frozen teacher, then freezes that harness and distills a 4B student, including training across multiple benchmarks and embodiments \citep{gong2026tinyastral}.
SenseNova-RoboRSI-27B-Preview further reports joint model--harness RSI from validated interaction experience, but its public materials do not specify the harness initialization or update schedule \citep{roborsi2026}.
We explicitly bootstrap from an empty task-level harness: a meta agent builds and revises the tools and skills, alternating these revisions with model training (Sec.~\ref{sec:training}).

\subsection{Harness--Model Co-Training}
\label{sec:related-cotraining}

Jointly optimizing the model and its harness is also studied in general-purpose language agents.
Co-Harness uses failure analysis to propose and validate harness edits, then fine-tunes the model on trajectories collected with the improved harness \citep{chen2026coharness}.
JoMoHa formulates the interaction as reinforcement learning between an Engineer that generates executable harnesses and a Solver that acts within them, training both from shared task outcomes \citep{wu2026jomoha}.
These methods show that the harness and model can adapt to one another.
Our focus is this interaction in physical robot control: tools and model weights evolve within a minimal robot harness.

\begin{figure}[tbp]
  \centering
  \includegraphics[width=\linewidth]{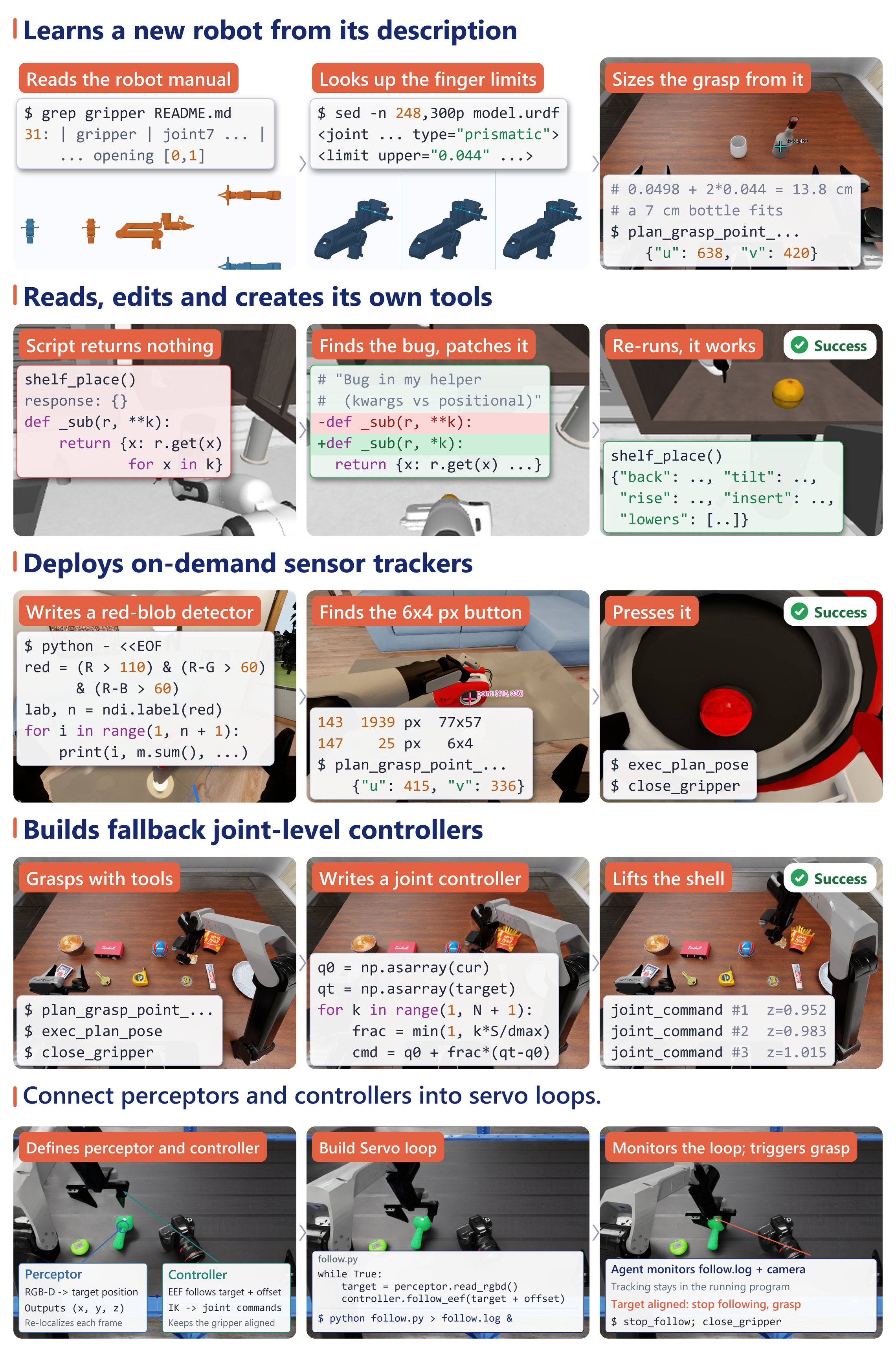}
  \caption{\textbf{Agent capabilities.} Rows 1--4 are recorded Qwen3.8-Flash-Next runs; row 5 is a schematic composed from recorded conveyor frames.}
  \label{fig:rollouts}
\end{figure}


\section{Problem Formulation}
\label{sec:formulation}

Let $e$ denote a robot embodiment (including its hardware structure and low-level joint command space $\mathcal{U}_e$), $t$ a task instruction, $h$ the initial workspace harness, and $\tau$ the resulting interaction trajectory.

\paragraph{Conventional robot agents.}
In existing frameworks, the hardware specifics of $e$ are hidden behind a fixed harness $h$. Trajectory generation reduces to a static conditional distribution:
\begin{equation}
    \tau \sim p_\theta^{\text{fixed}}(\tau \mid t, h), \quad \text{s.t. } \tau \in \mathcal{M}(h),
    \label{eq:fixed_agent}
\end{equation}
where $\mathcal{M}(h)$ is the restricted subset of physical motions reachable solely by calling the fixed toolset $h$ (Fig.~\ref{fig:action-space}).

\paragraph{\ours{}.}
Instead of coupling a fixed harness to the agent, \ours{} moves the harness to the task side: a physical robot task $(e, t)$ paired with its readable, editable workspace harness $h$ forms a single coding project ($\text{Task} = \text{task} + \text{harness}$, Fig.~\ref{fig:task-def}). A general, embodiment-agnostic coding agent $p_\theta$ takes this project $(e, t, h)$ as input and can not only invoke $h$, but also dynamically modify or create tools $h \to h'$. This expands generation into a joint distribution over $(\tau, h')$:
\begin{equation}
    (\tau, h') \sim p_\theta(\tau, h' \mid e, t, h) = \underbrace{p_\theta(h' \mid e, t, h)}_{\text{read } e \text{ \& edit/create } h'} \cdot \underbrace{p_\theta(\tau \mid e, t, h')}_{\text{execute } \tau \text{ under } h'}.
    \label{eq:roboharness}
\end{equation}
\paragraph{Why this formulation scales.}
Eq.~\ref{eq:roboharness} yields three direct properties (formal statements and proofs in Sec.~\ref{sec:supp-formulation}):
\begin{itemize}[leftmargin=*,nosep]
    \item \textbf{Higher reachability and return ceiling (Props.~\ref{prop:nested}--\ref{prop:edit}).}
    Higher abstraction shrinks reachable motions, $\mathcal{M}(h^{(L)})\subseteq\cdots\subseteq\mathcal{M}(h^{(0)})=\mathcal{M}(e)$, while joint commands $h^{(0)}$ cover all of $\mathcal{M}(e)$ (Fig.~\ref{fig:action-space}). Allowing edits $h\to h'$ therefore lifts the achievable ceiling for task return $R(\tau)$ from $\sup_{\mathcal{M}(h)}R$ (Eq.~\ref{eq:fixed_agent}) to $\sup_{\mathcal{M}(e)}R$ (Eq.~\ref{eq:roboharness}).
    \item \textbf{Zero-conflict mixing across robots and harnesses (Prop.~\ref{prop:mix}).}
    Conditioning on $(e,t,h)$ makes each $(e,t,h)$ optimal independently under SFT, so pooling different robots, harnesses, or co-training versions $h^{(0)}\!\to\!\cdots\!\to\!h^{(r)}$ causes zero conflict---whereas hiding $(e,h)$ suffers an irreducible excess loss $I(Y;E,H\mid T)\ge 0$.
    \item \textbf{Target-robot gain from extra robots and harnesses (Prop.~\ref{prop:aux}).}
    When domains share latent coding and 3D manipulation structure $\phi$, adding auxiliary data $D_2$ from other embodiments and harnesses lowers the target Bayes SFT loss by $I(Y;D_2\mid D_1)\ge 0$---both with target data $D_1$ present (multi-domain scaling, Fig.~\ref{fig:train-ablation}c--d) and with $D_1=\emptyset$ (zero-shot transfer, Fig.~\ref{fig:train-ablation}e--f).
\end{itemize}

\section{Method}
\label{sec:method}
RoboHarness is built on one design principle: keep the harness minimal, leaving more learning space for training to scale (full proofs in Sec.~\ref{sec:supp-formulation}).
Guided by this principle, we design RoboHarness as follows.

\paragraph{A robot is a project for a coding agent.}
Each episode starts in a workspace containing robot descriptions, tools, skills and observations (Fig.~\ref{fig:workspace}).
Each robot has its own directory and Python environment; switching embodiments means entering another project (Fig.~\ref{fig:task-def}).
Like an engineer taking over a new robot, the agent first studies the manual and a rendering of the robot to learn its structure, works out how its sensors are used, and reads the supporting scripts such as inverse kinematics.
It then writes and runs code and checks the results through native shell and file tools.

\begin{figure}[!htbp]
  \centering
  \includegraphics[width=0.855\linewidth]{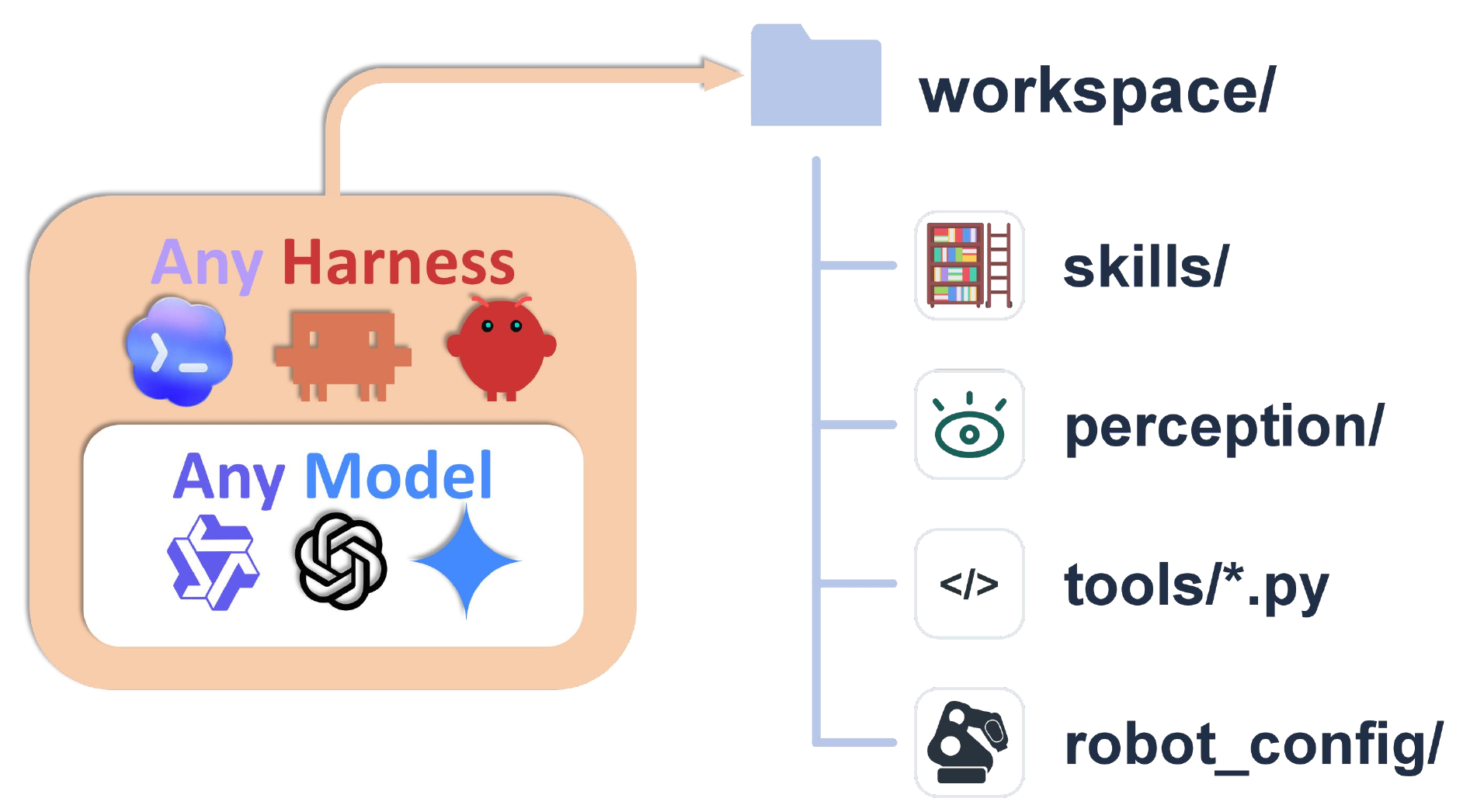}
  \caption{\textbf{An agent started inside the workspace.} Any coding-agent harness running any model is launched in \texttt{workspace/}, which contains \texttt{skills/}, \texttt{perception/}, \texttt{tools/*.py} and \texttt{robot\_config/}.}
  \label{fig:workspace}
\end{figure}

\paragraph{Minimal hardware contract.}
The thinner the hardware contract, the more motions it can express, and the better it serves agent training and generalization (Fig.~\ref{fig:action-space}; Prop.~\ref{prop:nested}).
Joint-level commands are the lowest-level interface, and any higher-level program can drive the robot through them.
Each robot therefore exposes only a joint-level command interface.
For example, we use a 27-value command for R1~Pro and 14 scalars for the dual-arm ARX~X5.

Different toolsets for the same robot share this low-level interface, while their upper layers remain readable to the agent.
An agent in \ours can thus generalize across toolsets, and it keeps learning when the toolset is revised during training; this is what enables the co-training in Sec.~\ref{sec:training}.

For different robots, \texttt{robot\_config/} labels every joint index and visualizes its direction of motion, so that the agent can learn to understand different embodiments.

\begin{figure}[!htbp]
  \centering
  \includegraphics[width=0.855\linewidth]{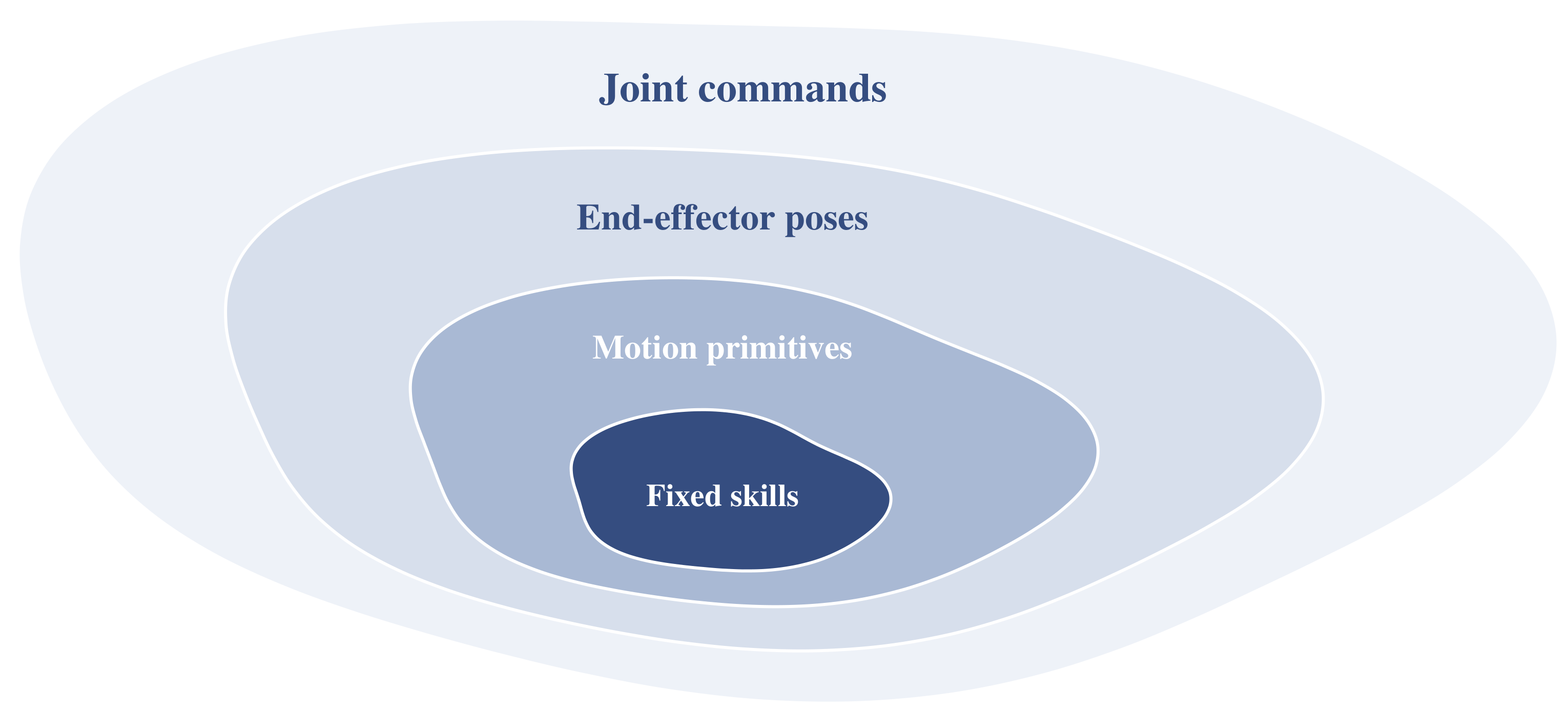}
  \caption{\textbf{Higher abstraction reaches fewer motions.} From joint commands and end-effector poses to motion primitives and fixed skills, each more abstract interface expresses a narrower range of motions.}
  \label{fig:action-space}
\end{figure}

\paragraph{Tools are a ladder, not a fence.}
A robot agent needs tools to control the robot.
We provide reusable code for perception, grasp planning, pose execution and inverse kinematics.
But tools should be a ladder, not a fence: each tool is a readable Python script with an input schema, and its implementation is exposed to the agent rather than hidden behind a call.
At run time the agent can therefore edit a tool and write ad-hoc control scripts of its own, raising the achievable return ceiling (Prop.~\ref{prop:edit}).
In Fig.~\ref{fig:rollouts} (row~2), a tool returns an empty response; the agent reads the script, finds a bug in a helper, fixes it in one line and reruns the tool successfully.

For example, a \texttt{grasp()} call passes through planning, pose execution and IK, and any of these steps can get stuck.
The agent may enter at any layer, revise its implementation, or write a replacement that commands joints directly.
In Fig.~\ref{fig:rollouts} (row~4), the named tools grasp the object, and the agent then writes its own short joint-level script to lift it.

\begin{figure}[tbp]
  \centering
  \includegraphics[width=\linewidth]{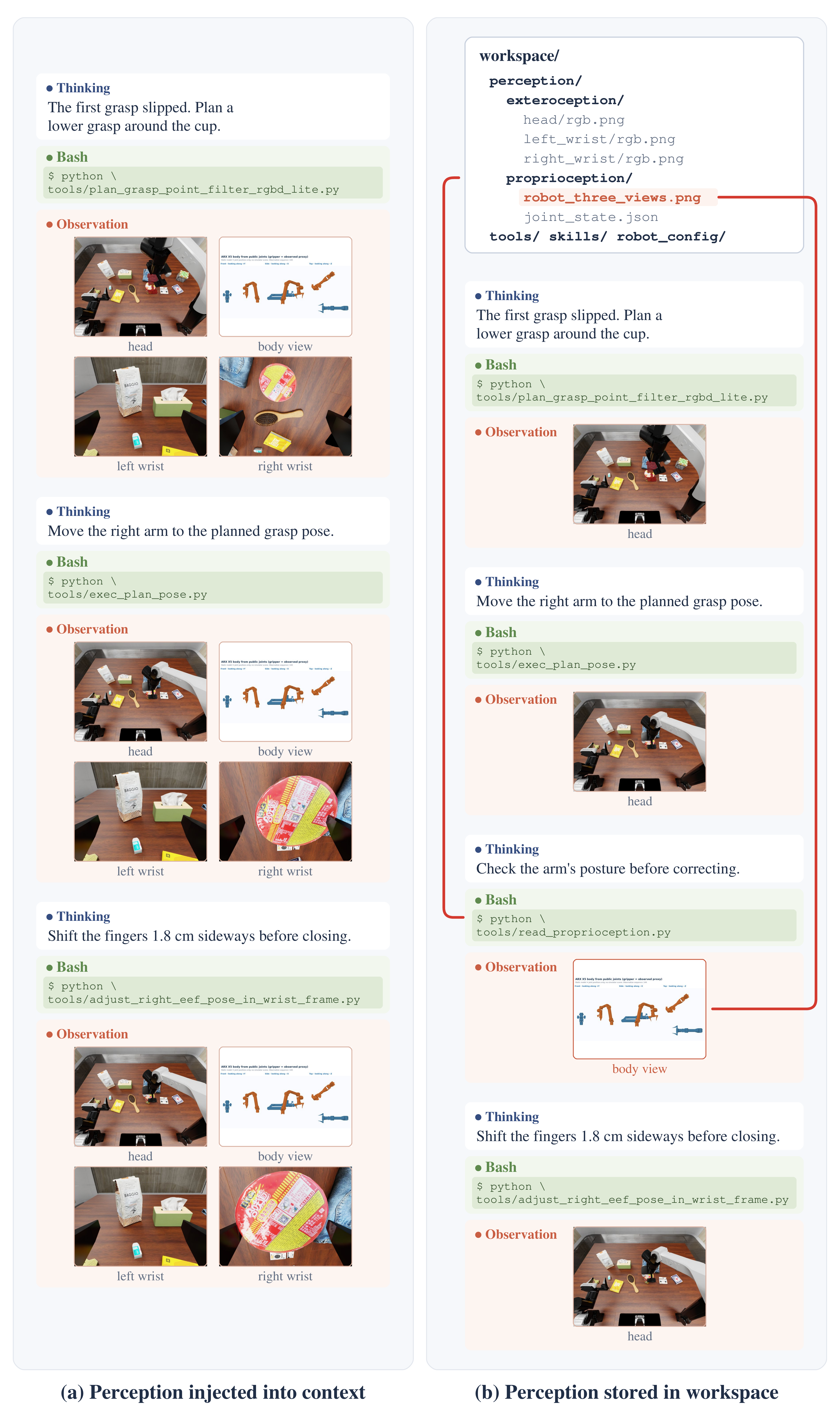}
  \caption{\textbf{Sensor streams on disk, read on demand.} }
  \label{fig:perception-files}
\end{figure}

\paragraph{Actions are background processes, not blocking calls.}
Long-running action scripts can be launched in the background through the coding agent's native shell.
The agent can continue reasoning, read available observations and inspect process output while a script runs, then wait for its completion when needed.
This separates the duration of physical execution from the duration of an agent turn and lets the agent choose when to check the outcome.

\paragraph{The sensor stream lives on disk.}
During a robotic task, the sensor streams keep growing while the useful information gets sparse.
If the stream is returned as tool results and appended to the context in a ReAct loop, the useful frames are buried among irrelevant ones and the model's attention is diluted (Fig.~\ref{fig:perception-files}, left).

We instead organize every sensor stream and archive it in the agent's workspace: RGB images, depth, tracked points and measured joint states are published under \texttt{perception/}, a \texttt{latest.json} manifest points to the newest files, and the workspace documents how to read them.
The agent reads only the views or state it needs, which saves a large amount of context (Fig.~\ref{fig:perception-files}, right).
For example, it inspects the head camera during grasp planning and requests a body view only when checking posture.
When reading is not enough, it can write a perceptor script that extracts the information for it (Fig.~\ref{fig:rollouts}, row~3).

This abstraction also carries over to any embodiment.
If a robot has an extra third-person camera, it only needs to be documented in \texttt{robot\_config/}, and the agent can use it without further changes.
This is one way \ours supports generalization across embodiments.

\paragraph{Connect perceptors and controllers directly into servo loops.}
Because the harness is a workspace of scripts, the agent can write perceptor and controller scripts itself and connect them directly through a servo program.
Simple, repetitive operations then run automatically in the servo loop while the agent only supervises, which avoids spending model calls on trivial decisions.
For example, to grasp an object on a moving conveyor, the agent need not adjust the end effector step by step to follow it.
Instead, a tracker sends the object's position to an end-effector controller that keeps the gripper at a suitable offset from the object (Fig.~\ref{fig:rollouts}, row~5).

\section{Experiments}
\label{sec:experiments}

Our experiments are designed to answer two core questions: (1) Does the \ours contract provide a sufficiently universal abstraction to achieve zero-shot generalization across heterogeneous embodiments, third-party toolsets, and scenes without training? (2) Does this unified abstraction unlock effective multi-benchmark scaling, where trajectories from diverse robots and toolsets benefit each other?

Section~\ref{sec:training} then explores post-training under \ours. We show that \ours unlocks \textbf{multi-task scaling}: training on a mixture of diverse embodiments and toolsets consistently improves performance over single-domain training, confirming that robot coding strategies learned on one embodiment can transfer to unseen hardware. 
Beyond static toolsets, \ours also shows the \textbf{potential for open-ended scaling} through harness--model co-training: the harness can be continuously refined to lift the performance ceiling, while the model keeps learning alongside the evolving toolset by reusing historical trajectories rather than starting over at each harness revision.
\subsection{Harness Expressiveness and Generality}
\label{sec:generalization}

\begin{table}[t]
\caption{\textbf{Porting harnesses into \ours.} $^{\dagger}$~Score is reported in the original paper. $^{\ddagger}$~Evaluated with matched 262{,}144-token context and truncated to the native per-instance turn limits (62\% uncapped; see Sec.~\ref{sec:supp-hexashell}).}
\label{tab:port}
\centering
\footnotesize
\setlength{\tabcolsep}{3pt}
\renewcommand{\arraystretch}{1}
\begin{tabular}{@{}llr@{}}
\toprule
Harness & Model & Score (\%) \\
\midrule
\rowcolor{black!7}\multicolumn{3}{@{}l}{\textit{RoboDojo (ARX~X5)}} \\
uni-robot-agent~\citep{yuan2026unirobotagent}, native$^{\dagger}$ & Qwen3.8-Flash-Next & 31.39 \\
uni-robot-agent, in \ours & Qwen3.8-Flash-Next & 30.78 \\
\midrule
\rowcolor{black!7}\multicolumn{3}{@{}l}{\textit{RoboCasa365 (Panda)}} \\
HexaShell~\citep{gao2026hexashell}, native & Qwen3.8-Flash-Next & 42 \\
HexaShell, in \ours$^{\ddagger}$ & Qwen3.8-Flash-Next & 48 \\
\midrule
\rowcolor{black!7}\multicolumn{3}{@{}l}{\textit{BEHAVIOR (R1~Pro)}} \\
\ours, 100 tasks & GPT-6 Astra & 45.7 \\
\ours w/ skill, subset & Qwen3.8-Flash-Next & 44.5 \\
\bottomrule
\end{tabular}
\end{table}


We choose three embodiments, one for each benchmark (RoboDojo~\citep{chen2026robodojo}, RoboCasa365~\citep{nasiriany2026robocasa365} and BEHAVIOR~\citep{li2024behavior}):
\begin{itemize}[leftmargin=*,nosep]
\item \textbf{RoboDojo, ARX~X5:} dual-arm tabletop manipulation;
\item \textbf{RoboCasa365, Panda:} single-arm mobile manipulation;
\item \textbf{BEHAVIOR, R1~Pro:} humanoid mobile manipulation.
\end{itemize}


For the toolset, we adapt two published harnesses to the \ours contract, uni-robot-agent~\citep{yuan2026unirobotagent} on RoboDojo and HexaShell~\citep{gao2026hexashell} on RoboCasa365, by writing their tools as files under \texttt{tools/}; detailed modifications are described in Sec.~\ref{sec:supp-static}.

Table~\ref{tab:port} reports the RoboDojo score, the RoboCasa365 success rate and the BEHAVIOR $q$-score.

On RoboCasa365 (with HexaShell's task-specific skills on the same 50 paired instances and a matched 262{,}144-token context window), HexaShell scores 42\% natively and 48\% inside \ours when truncated to the native per-instance turn limits after excluding workspace/CLI overhead turns (62\% without turn truncation; turn-alignment accounting is detailed in Sec.~\ref{sec:supp-hexashell}).
On BEHAVIOR (R1~Pro), we run \ours with GPT-6 Astra on all 100 tasks of the BEHAVIOR Challenge 2026 (one instance each) and reach a mean $q$-score of 45.7\%; harness and protocol details are in Sec.~\ref{sec:supp-behavior}.
With task-specific skills, Qwen3.8-Flash-Next reaches a mean $q$-score of 44.5\% on the first-10 subset of tasks (Sec.~\ref{app:behavior-instances}).
The per-task scores are in Sec.~\ref{app:object-transfer} of the supplementary material.

\paragraph{Results.} 
We show that the minimal contract of \ours is expressive: third-party toolsets can be ported into the \ours contract without degrading performance (within 0.61 points on RoboDojo and +6 points on RoboCasa365), demonstrating that \ours can express diverse tool implementations. Furthermore, this unified contract enables robot agents to control different robots across varied simulators and handle unseen objects and tasks.

\paragraph{100 Object General Picking Test.}\mbox{}\par\nopagebreak
\vspace{1em}

\noindent\begin{minipage}[c]{0.55\linewidth}
\textbf{Object selection.}
The catalog is the 100 rigid objects in Fig.~\ref{fig:objects}. Assets are taken from the BEHAVIOR object library \citep{li2024behavior}.
The objects vary in size and shape, and many can only be grasped from a specific angle at a specific location.
If a harness cannot express the required robot motions well enough, the robot gets stuck or fails.
\end{minipage}\hfill
\begin{minipage}[c]{0.42\linewidth}
\centering
\includegraphics[width=\linewidth]{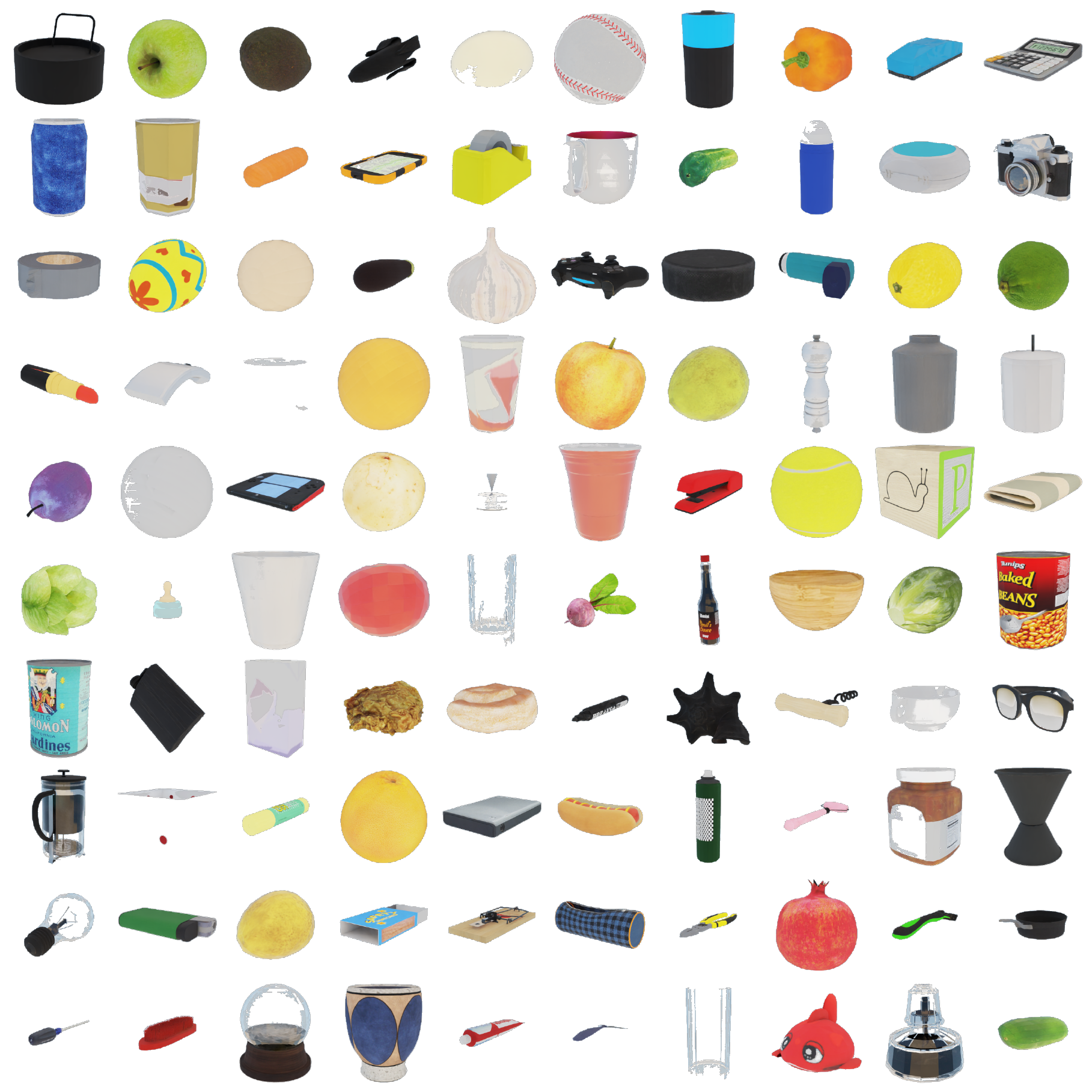}
\captionof{figure}{\textbf{100 test objects.} }
\label{fig:objects}
\end{minipage}

\textbf{Protocol.}
Each of the $100$ objects must be lifted by at least $5$\,cm within $2000$ simulation steps. Full settings are provided in Sec.~\ref{app:robot-hours}.

\textbf{Results.}
Compared with the fixed pick function from the code released with ASPIRE~\citep{lu2026aspire}, \ours gives the agent a larger space of actions and generalizes better across objects.
Although VLAs and WAMs are trained on more robot data, they lack the world knowledge of large language models and score no higher than the robot agents on general pickup.
The data accounting and the detailed scores are in Secs.~\ref{app:robot-hours} and \ref{app:object-transfer} of the supplementary material.

\begin{figure}[H]
\vspace{-6pt}
\centering
\includegraphics[width=0.82\linewidth]{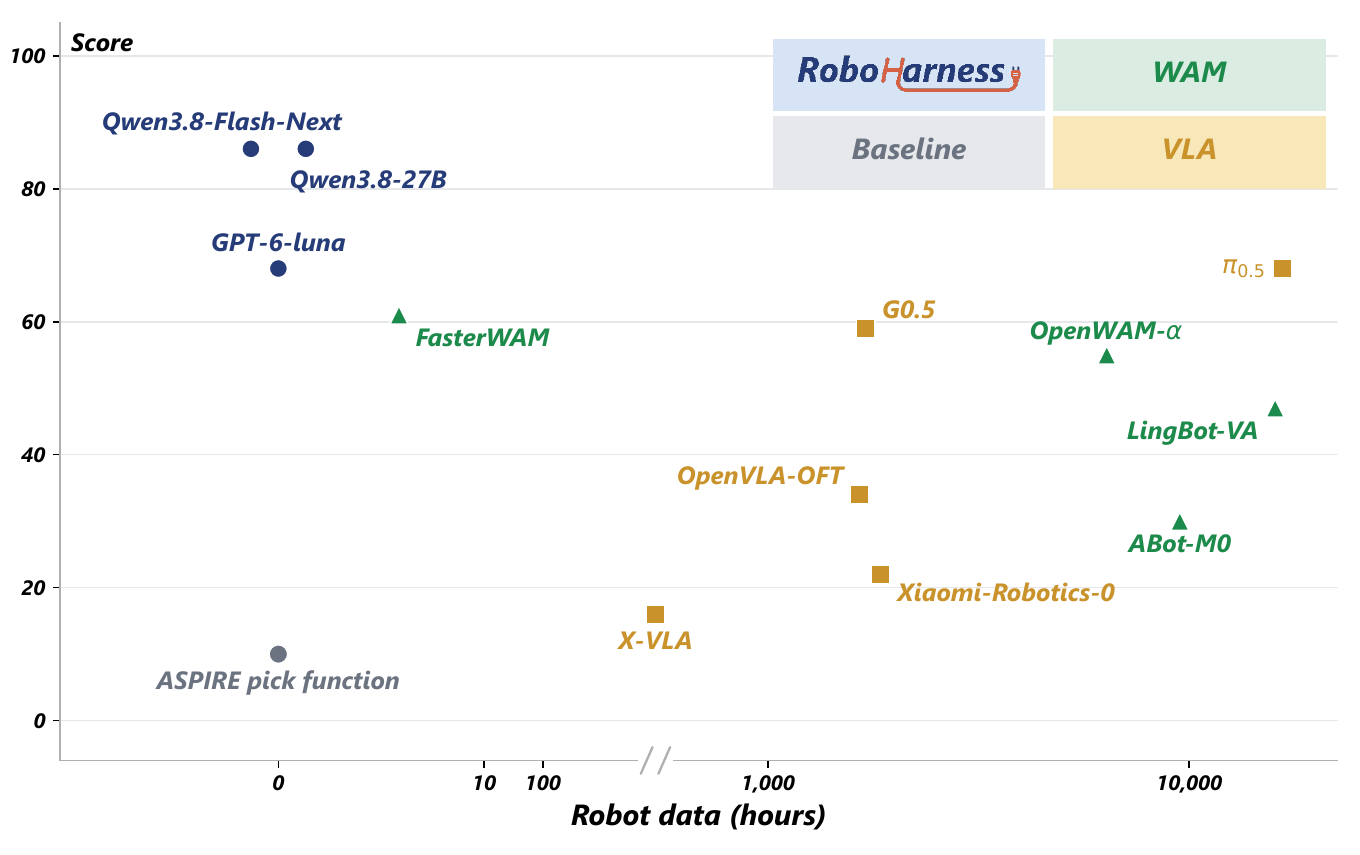}
\vspace{-4pt}
\caption{\textbf{100 task picking test performance.} Score is the number of successes out of 100 objects; the horizontal axis is an estimate of the amount of robot data used by each method.}
\label{fig:robot-hours}
\end{figure}

\subsection{RoboHarness Enables Training to Scale}
\label{sec:training}

\begin{figure*}[!t]
\centering
\ifdefined\arxivversion
\includegraphics[width=\textwidth]{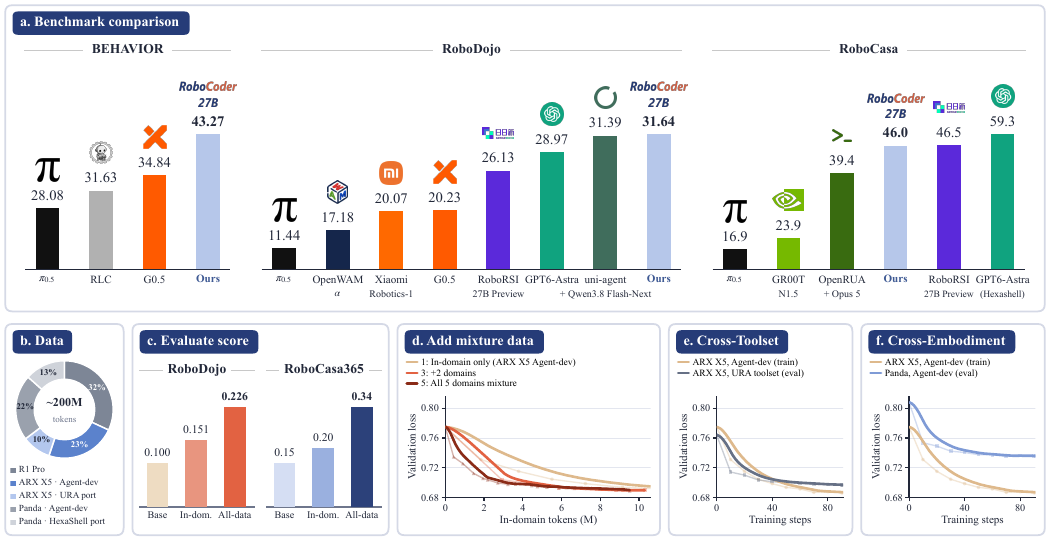}
\caption{\textbf{Benchmark comparisons, training mixture, and ablations.} \textbf{Top (a):} benchmark comparison of RoboCoder-27B against reference systems (Sec.~\ref{sec:supp-final}). \textbf{Bottom:} training-data mixture (b, Sec.~\ref{sec:supp-sft-data}), SFT evaluation scores on RoboDojo and RoboCasa365 (c, Sec.~\ref{sec:supp-vs-base}), adding mixture data (d), cross-toolset transfer (e), and cross-embodiment transfer (f). Curves are EMA-smoothed and faint lines show raw checkpoints.}
\label{fig:train-ablation}
\else
\includegraphics[width=\textwidth]{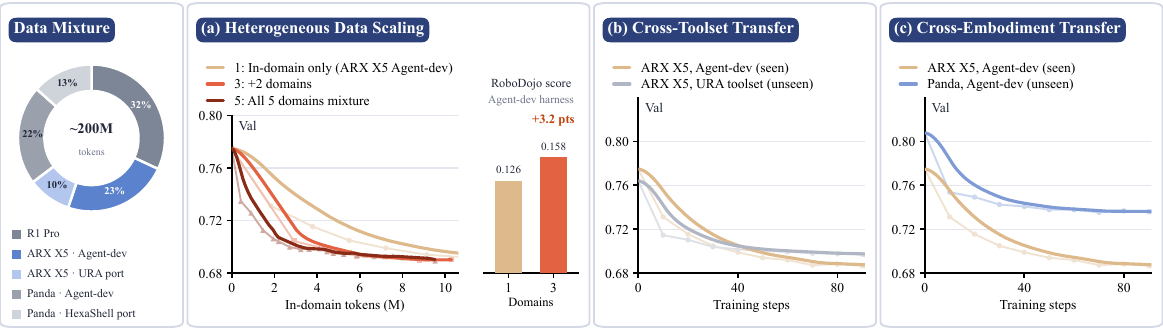}
\caption{\textbf{Training data and ablations.} \textbf{Left:} shares of total tokens in the training mixture (about 200M tokens; 39.8M loss tokens). \textbf{Right:} validation loss (Val); curves are EMA-smoothed and the faint lines are the raw checkpoints. \textbf{(a)}~Heterogeneous data scaling: models trained on 1, 3 or 5 of the five datasets are evaluated on the same ARX X5 Agent-dev validation set, against ARX X5 Agent-dev tokens seen. The bars on the right of (a) show the RoboDojo score under the Agent-dev harness, measured as the mean score over 42 ARX X5 tasks (no hint, one instance per task) for the 1-domain and 3-domain models; the 5-domain model has not been evaluated. \textbf{(b,c)}~Models trained only on ARX X5 Agent-dev are evaluated on the seen set and on an unseen toolset (b) or an unseen robot (c), against training step. Each validation set has 64 samples. The five-dataset curve covers the first epoch only.}
\label{fig:train-ablation}
\fi
\end{figure*}

\paragraph{Heterogeneous Data.}
We run Qwen3.8-Flash-Next on different embodiments through different \ours implementations, and collect a cross-embodiment SFT dataset of about 200M tokens (Fig.~\ref{fig:train-ablation}b; data generation and cleaning details in Sec.~\ref{sec:supp-sft-data}) in five parts:
\begin{itemize}[leftmargin=*,nosep]
\item \textbf{ARX X5, Agent-dev.} Following the \ours contract, Claude Code implements a harness for the ARX X5 from scratch, and we collect trajectories with it.
\item \textbf{ARX X5, URA.} We port the tools of uni-robot-agent (URA)~\citep{yuan2026unirobotagent} to \ours.
\item \textbf{Panda, Agent-dev.} The same procedure as for the ARX X5, applied to the Panda.
\item \textbf{Panda, HexaShell.} We port the tool implementation of HexaShell~\citep{gao2026hexashell} to \ours, which tests whether training generalizes across toolset implementations.
\item \textbf{R1 Pro.} The harness details are in Sec.~\ref{sec:supp-behavior}.
\end{itemize}
All trajectories are collected by running Qwen3.8-Flash-Next inside the \ours workspace; ``URA'' and ``HexaShell'' denote third-party toolsets ported into \ours, not third-party data.

\paragraph{Data Ablations.}
We write an embodied instance as a triple $(e,t,h)$ of embodiment, task and harness, the policy as $\pi_\theta(\tau\mid e,t,h)$, and the trajectories of an instance as $\mathcal D(e,t,h)$. We conduct three controlled training experiments to study how capabilities learned across heterogeneous datasets, toolsets, and embodiments transfer under \ours (Fig.~\ref{fig:train-ablation}). All validation sets are frozen and held out from training.
\begin{itemize}[leftmargin=*,nosep]
    \item \textbf{Multi-domain scaling.} We optimize $\theta$ on $\bigcup_i\mathcal D(e_i,t_i,h_i)$ and evaluate on the fixed target $(e^*,t^*,h^*)$ = ARX~X5 Agent-dev. We compare (i)~the target domain alone (\emph{1~domain}), (ii)~adding Panda Agent-dev and R1~Pro (\emph{3~domains}), and (iii)~all five datasets, further adding the URA and HexaShell subsets (\emph{5~domains}) (Fig.~\ref{fig:train-ablation}d).
    \item \textbf{Cross-toolset transfer.} On the same ARX~X5 robot, we train $\theta$ on $(e,t,h_1)$ with the Agent-dev toolset and evaluate $\pi_\theta(\cdot\mid e,t,h_2)$ on the URA toolset $h_2$, whose implementation the model never sees during training (Fig.~\ref{fig:train-ablation}e).
    \item \textbf{Cross-embodiment transfer.} We train $\theta$ on the dual-arm tabletop ARX~X5, $(e_A,t_A,h_A)$ with Agent-dev, and evaluate $\pi_\theta(\cdot\mid e_B,t_B,h_B)$ on the single-arm mobile Panda with Agent-dev (Fig.~\ref{fig:train-ablation}f).
\end{itemize}

\paragraph{Results.}
Pooling heterogeneous data lowers the validation loss on the target domain faster per ARX~X5 token and to a lower floor than single-domain training (Fig.~\ref{fig:train-ablation}d), and the RoboDojo score under the Agent-dev harness rises from $0.126$ for the 1-domain model to $0.158$ for the 3-domain model, a gain of $3.2$ points.
Training on ARX~X5 Agent-dev alone lowers the validation loss relative to the base model by $8.8\%$ on the unseen URA toolset (Fig.~\ref{fig:train-ablation}e) and by $8.9\%$ on the unseen Panda embodiment (Fig.~\ref{fig:train-ablation}f).
This indicates that data pooled across $(e,t,h)$ benefits a single target domain, and that what is learned under one toolset or embodiment transfers to another (Props.~\ref{prop:mix}--\ref{prop:aux}).
Held-out SFT evaluation ablations (Fig.~\ref{fig:train-ablation}c, Sec.~\ref{sec:supp-vs-base}), final RoboCoder-27B benchmark comparisons (Fig.~\ref{fig:train-ablation}a, Sec.~\ref{sec:supp-final}), and per-task score tables (Sec.~\ref{app:object-transfer}) are detailed in the supplementary material.

\paragraph{Harness--model co-training.}
Starting from Qwen3.8-27B and an empty harness, we design an iterative pipeline to co-train the harness and the model (Fig.~\ref{fig:cotrain}).
The initial workspace provides only robot documentation, sensor access, and the raw hardware contract, with no predefined task tools or skills.
A meta-agent builds the initial tools and skills from scratch and revises them across rounds without human-designed task priors.

The initial data are produced by Qwen3.8-Flash-Next, and each later training stage adds the data produced by the model of the previous stage.

Throughout training, the meta-agent runs asynchronously alongside data production and monitors low-scoring tasks. It analyzes the failure causes and interactively re-solves these tasks; once it succeeds, it packages the resulting tools and skills and runs the latest policy model on this harness to produce successful trajectories, which fill in the data distribution for hard tasks. These runs are not counted in the scores of Fig.~\ref{fig:cotrain-curve}.

\begin{figure*}[!t]
\centering
\includegraphics[width=0.92\textwidth]{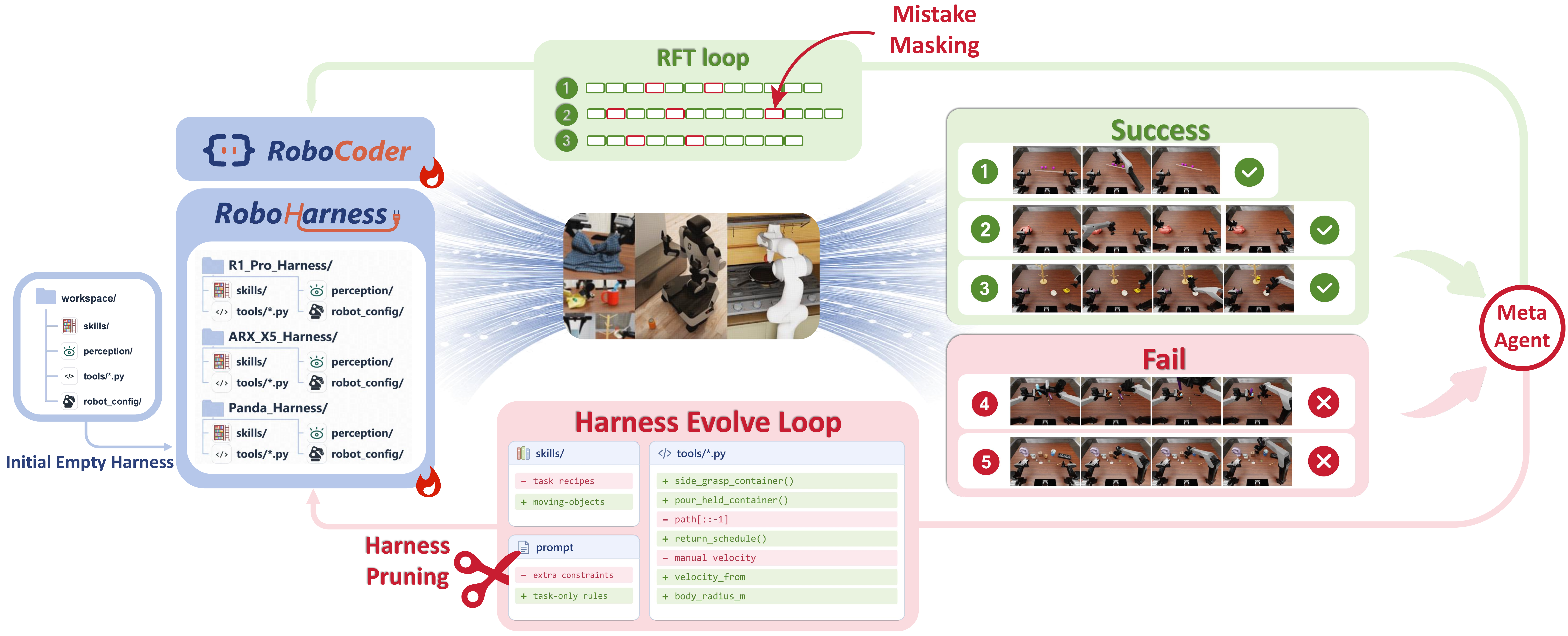}
\caption{\textbf{Co-training model and harness:} Initialized from an empty workspace, a meta-agent builds and develops the harness by itself under the \ours contract. Successful rollouts train the model while the meta-agent updates the harness using successes and failures.}
\label{fig:cotrain}
\end{figure*}

Since successful rollouts can still contain suboptimal patterns, the meta-agent identifies and filters them out (mistake masking, Fig.~\ref{fig:cotrain}) before training.

Evolving the harness can be expressed as $h^{(0)} \to \dots \to h^{(r)}$ in $p_\theta(\tau, h' \mid e, t, h^{(j)})$ (Eq.~\ref{eq:roboharness}). Since our cross-toolset experiments (Fig.~\ref{fig:train-ablation}d,e) establish that trajectories from distinct toolsets on the same $e$ mutually reinforce one another, historical rollouts $\bigcup_{j=0}^{r} \mathcal{D}(e, t, h^{(j)})$ remain fully compatible and can be pooled directly to train $\theta^{(r+1)}$ without data invalidation (Prop.~\ref{prop:mix}).

After each evaluation, the meta-agent analyzes the successful and failed trajectories, applies harness pruning (removing part of the task-specific procedures, Fig.~\ref{fig:cotrain}) to avoid harness overfitting, revises part of the prompt and skills, and evaluates the model again with the model checkpoint unchanged (orange segments in Fig.~\ref{fig:cotrain-curve}).

\paragraph{Results.}
As shown in Fig.~\ref{fig:cotrain-curve} (per-step harness diffs in Sec.~\ref{sec:supp-cotrain}), harness improvements produce higher-quality trajectories, which in turn train a stronger model that makes subsequent harness updates easier.
Overall, the progression of the RoboDojo score $0.2012$ demonstrates the \textbf{potential for open-ended scaling} under \ours: the harness can iterate continuously toward better toolsets, and the model can keep learning alongside these upgrades by accumulating experience across versions.

\begin{figure}[H]
\centering
\includegraphics[width=\linewidth]{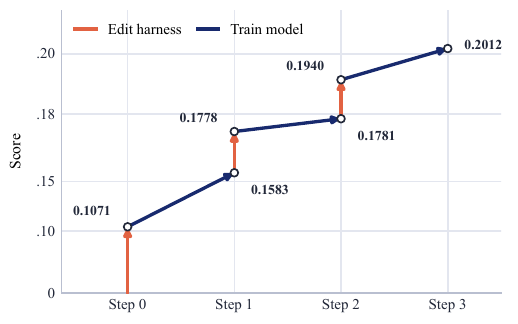}
\caption{\textbf{Harness--model co-training curve.} Step $k$ is the model checkpoint after the $k$-th training stage (Step~0 is the base model). Each stage changes only one thing: in a harness edit (orange), the model checkpoint is held fixed and only the harness changes; in a training stage (navy), the prompt and skills are held fixed and only the model is trained, while the toolset follows the training data. }
\label{fig:cotrain-curve}
\end{figure}

\section{Limitation \& Discussion}
\label{sec:discussion}

\ours demonstrates that a minimal workspace contract both expresses diverse robot control without training and shows the potential for multi-task and open-ended scaling in post-training.

Two main limitations remain. First, deploying \ours on physical hardware requires addressing real-time inference latency and high-frequency execution constraints. Second, our 200M-token mixture gives the evidence of data scaling, and more data is needed to scale it up.
\section{LLM Usage Statement}
Generative AI tools (Claude Code, Codex and Gemini) were used in writing, citation checking, coding, and experiment execution. The proofs in Secs.~\ref{sec:supp-formulation} and \ref{sec:proof-completeness} were written by the authors together with Gemini, and Claude revised their typesetting. Claude and Codex coding agents ran the experiments in the main paper and supplementary material; the authors checked every test setting and reported score. The adaptation of the third-party harnesses described in Sec.~\ref{sec:supp-static} was done by Codex, which also summarized that adaptation; the authors checked the code. AI-assisted text was refined by the authors against the source outline, and AI-assisted experimental results were checked against the underlying evaluator outputs. Coding agents are also an object of study in this paper: the evaluated systems are LLM coding agents operating the robot harness.

\ifdefined\withsuppl
\ifdefined\arxivversion\else
\clearpage
\fi
\ifdefined\arxivversion\else
\setcounter{page}{1}
\fi
\setcounter{section}{0}
\setcounter{figure}{0}
\setcounter{table}{0}
\setcounter{equation}{0}
\renewcommand{\thesection}{S\arabic{section}}
\renewcommand{\thefigure}{S\arabic{figure}}
\renewcommand{\thetable}{S\arabic{table}}
\renewcommand{\theequation}{S\arabic{equation}}
\renewcommand{\theHsection}{supp.\arabic{section}}
\renewcommand{\theHsubsection}{supp.\arabic{section}.\arabic{subsection}}
\renewcommand{\theHsubsubsection}{supp.\arabic{section}.\arabic{subsection}.\arabic{subsubsection}}
\renewcommand{\theHfigure}{supp.\arabic{figure}}
\renewcommand{\theHtable}{supp.\arabic{table}}
\renewcommand{\theHequation}{supp.\arabic{equation}}
\makeatletter
\renewcommand{\maketitlesupplementary}{%
  \ifdefined\arxivversion
    \clearpage
    \onecolumn
    \begingroup
      \raggedright
      \vspace*{0.5em}
      \newcommand{\tocsecgap}{10.0pt}\newcommand{\tocsubgap}{4.0pt}
      \noindent{\Large\bfseries Supplementary Material}\par
      \vspace{0.5em}
      {\setcounter{tocdepth}{2}
       \renewcommand\l@section[2]{\ifnum\c@tocdepth>0\vskip\tocsecgap{\large\bfseries\@dottedtocline{1}{0em}{2.2em}{##1}{##2}}\fi}
       \renewcommand\l@subsection[2]{\ifnum\c@tocdepth>1\vskip\tocsubgap{\normalsize\@dottedtocline{2}{1.4em}{2.6em}{##1}{##2}}\fi}
       \@starttoc{toc}}
    \endgroup
    \clearpage
    \twocolumn
  \else
    \newpage
    \twocolumn[%
      \centering
      \Large
      \textbf{\thetitle}\\
      \vspace{0.5em}Supplementary Material \\
      \vspace{1.0em}
      \raggedright
      \vspace{0.5em}
      \newcommand{\tocsecgap}{10.0pt}\newcommand{\tocsubgap}{4.0pt}
      \noindent{\Large\bfseries Contents}\par
      {\setcounter{tocdepth}{2}
       \renewcommand\l@section[2]{\ifnum\c@tocdepth>0\vskip\tocsecgap{\large\bfseries\@dottedtocline{1}{0em}{2.2em}{##1}{##2}}\fi}
       \renewcommand\l@subsection[2]{\ifnum\c@tocdepth>1\vskip\tocsubgap{\normalsize\@dottedtocline{2}{1.4em}{2.6em}{##1}{##2}}\fi}
       \@starttoc{toc}}
    ]%
  \fi
}
\makeatother
\maketitlesupplementary
\addtocontents{toc}{\protect\setcounter{tocdepth}{2}}
\ifdefined\arxivversion\else
\clearpage
\fi
\sloppy
\providecommand{\dm}[1]{{\color{red}#1}}
\providecommand{\up}[1]{{\color[rgb]{0.08,0.50,0.22}\textbf{#1}}}
\providecommand{\down}[1]{{\color[rgb]{0.78,0.12,0.12}\textbf{#1}}}

\ifdefined\arxivversion
  \newcommand{\figsuppref}{Fig.~\ref{fig:train-ablation}}
  \newcommand{\figsupppanel}[1]{Fig.~\ref{fig:train-ablation}#1}
\else
  \newcommand{\figsuppref}{Fig.~\ref{fig:supp-final}}
  \newcommand{\figsupppanel}[1]{Fig.~\ref{fig:supp-final}#1}
\fi

\ifdefined\arxivversion\else
\begin{figure*}[t]
  \centering
  \includegraphics[width=\textwidth]{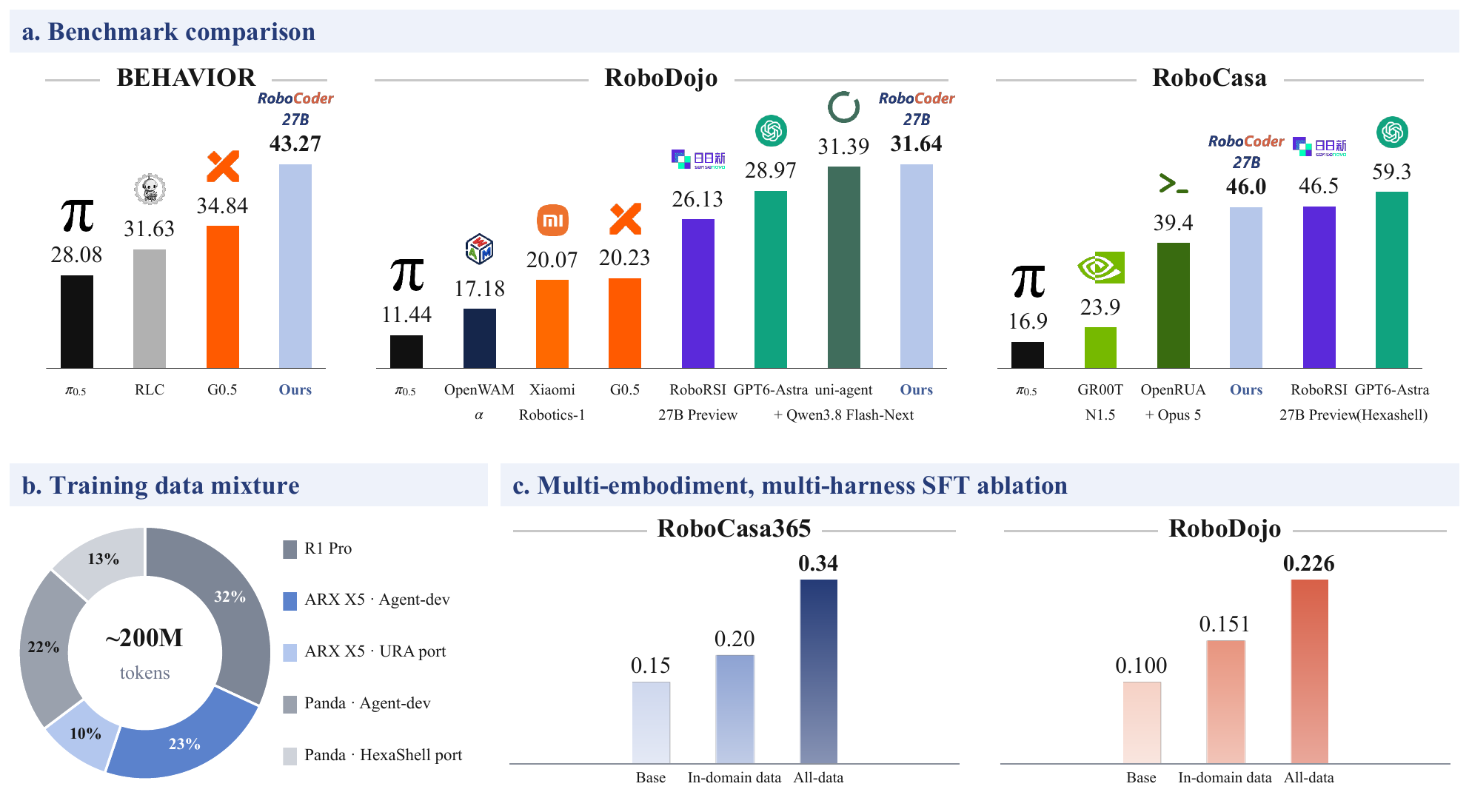}
  \caption{\textbf{Benchmark comparison of RoboCoder-27B} \textbf{Top:} benchmark comparison. \textbf{Bottom left:} share of different embodiment and harnesses in the SFT mixture (about 200M tokens). \textbf{Bottom right:} multi-embodiment, multi-harness SFT ablation on RoboCasa365 and RoboDojo. RoboDojo is tested on agent-dev RoboHarness, and RoboCasa is tested on Hexashell-port RoboHarness. See Sec.~\ref{app:object-transfer} for details.}
  \label{fig:supp-final}
\end{figure*}
\fi

\section{Proofs: Minimal Coding Harness, General Robot Agent}
\label{sec:supp-formulation}

\paragraph{Preliminaries.}
\emph{Trajectories and toolsets:} For embodiment $e$, $\mathcal{U}_e^{*}$ denotes finite joint-command sequences, $\Phi_e$ maps $\mathbf{u}\in\mathcal{U}_e^{*}$ to trajectory $\tau=\Phi_e(\mathbf{u})$, and $\mathcal{M}(e):=\Phi_e(\mathcal{U}_e^{*})$ is all reachable trajectories. A toolset $h:\mathcal{A}_h\to\mathcal{U}_e^{*}$ compiles calls $\mathcal{A}_h$ into joint commands, reaching $\mathcal{M}(h):=\Phi_e(h(\mathcal{A}_h))\subseteq\mathcal{M}(e)$ (nonempty); $\mathcal{H}_e$ is all expressible toolsets. On task $t$, trajectory $\tau$ has return $R(\tau)\in[0,1]$, and policy $p$ has expected return $J(p):=\mathbb{E}_{\tau\sim p}[R(\tau)]$.

\noindent\emph{SFT loss:} For countable $(C,Y)$ with finite $H(Y\mid C):=\mathbb{E}[-\ln P(Y\mid C)]$, SFT minimizes $\mathcal{L}(p):=\mathbb{E}[-\ln p(Y\mid C)]$. For distributions $P_0,r$ (with $a\ln(a/0):=+\infty$ for $a>0$), Gibbs' inequality ($\ln u\le u-1$ at $u=r(y)/P_0(y)$) yields $D_{\mathrm{KL}}(P_0\|r):=\sum_{P_0(y)>0}P_0(y)\ln\frac{P_0(y)}{r(y)}\ge0$ ($=0$ iff $r=P_0$), so
\begin{equation}\label{eq:ce-kl}
\mathcal{L}(p) = H(Y\mid C)+\mathbb{E}_C\big[D_{\mathrm{KL}}(P(\cdot\!\mid\!C)\|p(\cdot\!\mid\!C))\big],
\end{equation}
giving $\mathcal{L}(p)\ge H(Y\mid C)$, uniquely minimized (on $P(C)>0$) at $p(\cdot\mid C)=P(\cdot\mid C)$.

\begin{proposition}\label{prop:nested}%
{\normalfont\bfseries Lower abstraction covers more motions, with joint commands covering all (Fig.~\ref{fig:action-space}).}\addcontentsline{toc}{subsection}{Prop.~1.~Lower abstraction covers more motions, with joint commands covering all}\mbox{}\\*
Let $\mathcal{A}_0:=\mathcal{U}_e^{*}$, layer $k\in\{1,\dots,L\}$ map calls $\mathcal{A}_k$ via $f_k:\mathcal{A}_k\to\mathcal{A}_{k-1}$, and $h^{(k)}:=f_1\circ\cdots\circ f_k$ ($h^{(0)}=\mathrm{id}$):
\[
\mathcal{M}(h^{(L)})\subseteq\cdots\subseteq\mathcal{M}(h^{(1)})\subseteq\mathcal{M}(h^{(0)})=\mathcal{M}(e),
\]
with strict $k$-th inclusion whenever layer $k$ excludes a trajectory reachable at layer $k-1$ (in particular, if $\Phi_e$ is injective and $h^{(k)}(\mathcal{A}_k)\subsetneq h^{(k-1)}(\mathcal{A}_{k-1})$).
\end{proposition}
\begin{proof}
Since $f_k(\mathcal{A}_k)\subseteq\mathcal{A}_{k-1}$, applying $h^{(k-1)}$ and $\Phi_e$ yields $h^{(k)}(\mathcal{A}_k)\subseteq h^{(k-1)}(\mathcal{A}_{k-1})$ and $\mathcal{M}(h^{(k)})\subseteq\mathcal{M}(h^{(k-1)})$; at $k=0$, $\mathcal{M}(h^{(0)})=\Phi_e(\mathcal{U}_e^{*})=\mathcal{M}(e)$. If $\Phi_e$ is injective and $\mathbf{u}_0\in h^{(k-1)}(\mathcal{A}_{k-1})\setminus h^{(k)}(\mathcal{A}_k)$, no $\mathbf{u}\in h^{(k)}(\mathcal{A}_k)$ maps to $\Phi_e(\mathbf{u}_0)$, so $\Phi_e(\mathbf{u}_0)\in\mathcal{M}(h^{(k-1)})\setminus\mathcal{M}(h^{(k)})$.
\end{proof}

\begin{proposition}\label{prop:edit}%
{\normalfont\bfseries Editable tools yield a higher achievable return ceiling than fixed tools (Eqs.~\ref{eq:fixed_agent}--\ref{eq:roboharness}).}\addcontentsline{toc}{subsection}{Prop.~2.~Editable tools yield a higher achievable return ceiling than fixed tools}\mbox{}\\*
Fix $e$, $t$, and initial toolset $h$. Suppose $\mathcal{H}_e$ contains $h$ and a joint-command script $h_{\mathrm{id}}$ with $\mathcal{M}(h_{\mathrm{id}})=\mathcal{M}(e)$. Over all policies, the highest achievable expected return is $\sup_{\mathcal{M}(h)}R$ under Eq.~\ref{eq:fixed_agent} (fixed $h$) and $\sup_{\mathcal{M}(e)}R$ under Eq.~\ref{eq:roboharness} (choosing $h'\in\mathcal{H}_e$), so Eq.~\ref{eq:roboharness}'s ceiling is at least Eq.~\ref{eq:fixed_agent}'s, and strictly higher if some $\tau^{*}\in\mathcal{M}(e)\setminus\mathcal{M}(h)$ has $R(\tau^{*})>\sup_{\mathcal{M}(h)}R$.
\end{proposition}
\begin{proof}
Any distribution $p$ on $S$ obeys $J(p)=\mathbb{E}_{\tau\sim p}[R(\tau)]\le\sup_S R$, while point mass $\delta_\tau$ achieves $R(\tau)$ for each $\tau\in S$; thus the ceiling on $S$ is $\sup_S R$. Eq.~\ref{eq:fixed_agent} reaches $S=\mathcal{M}(h)$. Under Eq.~\ref{eq:roboharness}, $\bigcup_{h'\in\mathcal{H}_e}\mathcal{M}(h')\subseteq\mathcal{M}(e)$, while $h'=h_{\mathrm{id}}$ attains $\mathcal{M}(e)$, so $S=\mathcal{M}(e)$. As $\mathcal{M}(h)\subseteq\mathcal{M}(e)$, $\sup_{\mathcal{M}(h)}R\le\sup_{\mathcal{M}(e)}R$; if $R(\tau^{*})>\sup_{\mathcal{M}(h)}R$, then $\sup_{\mathcal{M}(e)}R\ge R(\tau^{*})>\sup_{\mathcal{M}(h)}R$.
\end{proof}
\noindent\emph{Remark.} This compares ceilings; for a fixed model, editing helps when it beats keeping $h$ on average, motivating SFT (Sec.~\ref{sec:training}).

\begin{proposition}\label{prop:mix}%
{\normalfont\bfseries Mixing different robots or harnesses causes no conflict (Figs.~\ref{fig:train-ablation}c--d, \ref{fig:cotrain}--\ref{fig:cotrain-curve}).}\addcontentsline{toc}{subsection}{Prop.~3.~Mixing different robots or harnesses causes no conflict}\mbox{}\\*
Let $(E,T,H,Y)\sim P$ be a pooled SFT dataset across embodiments $E$, tasks $T$, harnesses (or versions) $H$, and outputs $Y$, with finite conditional entropies. Over all workspace-conditioned policies $p(y\mid e,t,h)$ (Eq.~\ref{eq:roboharness}):
(i)~$\mathcal{L}(p)$ is uniquely minimized at $p^{*}(y\mid e,t,h)=P(y\mid e,t,h)$ for every $(e,t,h)$ with $P(e,t,h)>0$---matching single-$(e,t,h)$ training, so pooling versions $h^{(0)}\!\to\!\cdots\!\to\!h^{(r)}$ in co-training causes zero conflict;
(ii)~versus any context-blind $q(y\mid t)$, conditioning on $(e,h)$ cuts the minimum loss by $I(Y;E,H\mid T)\ge0$ ($>0$ when optimal behavior depends on $(e,h)$).
\end{proposition}
\begin{proof}
By \eqref{eq:ce-kl} with $C=(E,T,H)$, the excess loss is a $P(e,t,h)$-weighted average of $D_{\mathrm{KL}}(P(\cdot\mid e,t,h)\,\|\,p(\cdot\mid e,t,h))$, so each $(e,t,h)$ term is minimized independently at $P(\cdot\mid e,t,h)$, proving~(i). For~(ii), \eqref{eq:ce-kl} with $C=T$ gives $\min_q\mathcal{L}(q)=H(Y\mid T)$; evaluating \eqref{eq:ce-kl} with $C=(E,T,H)$ at $p(\cdot\mid e,t,h)=P(\cdot\mid t)$ gives $H(Y\mid T)-H(Y\mid E,T,H)=\mathbb{E}[D_{\mathrm{KL}}(P(\cdot\mid E,T,H)\|P(\cdot\mid T))]=I(Y;E,H\mid T)\ge0$, strictly positive iff $P(\cdot\mid E,T,H)\ne P(\cdot\mid T)$ with positive probability.
\end{proof}

\begin{proposition}\label{prop:aux}%
{\normalfont\bfseries Extra data from other embodiments and harnesses improves target-robot performance (Fig.~\ref{fig:train-ablation}c--f, Sec.~\ref{sec:supp-vs-base}).}\addcontentsline{toc}{subsection}{Prop.~4.~Extra data from other embodiments and harnesses improves target-robot performance}\mbox{}\\*
For a fixed target query, let discrete latent $\phi$ capture shared cross-domain structure (e.g., coding idioms, 3D manipulation), with target data $D_1$ (possibly empty), auxiliary data $D_2$, and target output $Y$ conditionally independent given $\phi$ (finite conditional entropies). Comparing Bayes-optimal $p_1^{*}(y\mid D_1)=P(y\mid D_1)$ and $p_{12}^{*}(y\mid D_1,D_2)=P(y\mid D_1,D_2)$ gives $\mathcal{L}(p_1^{*})-\mathcal{L}(p_{12}^{*})=H(Y\mid D_1)-H(Y\mid D_1,D_2)=I(Y;D_2\mid D_1)\ge0$, strictly positive iff $P(\cdot\mid D_1,D_2)\ne P(\cdot\mid D_1)$ with positive probability; since $P(y\mid D_1,D_2)=\sum_\phi P(y\mid\phi)P(\phi\mid D_1,D_2)$, this occurs iff $D_2$ updates the posterior of $\phi$ and changes $Y$'s predictions.
\end{proposition}
\begin{proof}
By \eqref{eq:ce-kl}, $\mathcal{L}(p_1^{*})=H(Y\mid D_1)$ and $\mathcal{L}(p_{12}^{*})=H(Y\mid D_1,D_2)$. Evaluating \eqref{eq:ce-kl} with $C=(D_1,D_2)$ at $p(\cdot\mid D_1,D_2)=P(\cdot\mid D_1)$ gives $\mathcal{L}(p_1^{*})-\mathcal{L}(p_{12}^{*})=\mathbb{E}[D_{\mathrm{KL}}(P(\cdot\mid D_1,D_2)\|P(\cdot\mid D_1))]=I(Y;D_2\mid D_1)\ge0$, strictly positive iff $P(\cdot\mid D_1,D_2)\ne P(\cdot\mid D_1)$ with positive probability. Conditional independence $Y\perp(D_1,D_2)\mid\phi$ yields the marginalization identity (and for $D_1$ alone).
\end{proof}

\section{SFT}
\label{sec:supp-sft}

\subsection{SFT Data Generation and Filter}
\label{sec:supp-sft-data}

\paragraph{Generation.}
We generate data on the 100 tasks of BEHAVIOR, the 50 tasks of RoboCasa365 and the 42 tasks of RoboDojo. We exclude the test instances from the data: instance 301 of BEHAVIOR, the kitchen layouts $(1,1),(2,2),\dots,(10,10)$ of RoboCasa365, and instance 0 of seeds 0, 1 and 2 of RoboDojo. On all other instances we run Qwen3.8-Flash-Next in the \ours workspaces in large numbers and keep only the trajectories that receive the full official score. We checked that none of the kept trajectories falls in the excluded test set. The data come from five sources: the Agent-dev harnesses of BEHAVIOR, RoboCasa365 and RoboDojo, the HexaShell port on RoboCasa365, and the uni-robot-agent port on RoboDojo (Sec.~\ref{sec:supp-static}). Each trajectory is cut into training fragments of at most 81{,}920 tokens, the training length; for the agent-dev and HexaShell sources a fragment is a context prefix or a segment, and for uni-robot-agent a fragment is the span between two summaries. Only the model's own replies carry loss.

\paragraph{Cleaning.}
Cleaning has two steps.
\begin{enumerate}[leftmargin=*,nosep]
\item \textbf{Remove overlong data.}
  \begin{itemize}[leftmargin=*,nosep]
  \item \emph{Agent-dev (BEHAVIOR, RoboCasa365, RoboDojo):} a rollout is dropped if its number of assistant replies exceeds twice the median of the same task. This removes 48 of 685 full-score rollouts (BEHAVIOR 9, RoboCasa 9, RoboDojo 30). No fragment exceeds the length limit, so none is truncated.
  \item \emph{HexaShell-port:} for each task, if there are more than 10 full-score rollouts, we keep only the shortest one for each kitchen; with 5 to 10 we keep the shortest two; with fewer than 5 we keep all.
  \item \emph{uni-robot-agent-port:} all 65 full-score rollouts are kept, but a fragment longer than 81{,}920 tokens is dropped whole, never truncated. This drops 194 of 910 fragments (21.3\%).
  \end{itemize}
\item \textbf{Mask bad tool calls.} We identify replies whose tool call is wrong and mask them. A masked reply stays in the context so that the following steps remain coherent, but it carries no loss. The rules are applied automatically, and we manually verified samples of each class.
  \begin{itemize}[leftmargin=*,nosep]
  \item \emph{Agent-dev:} failed replies are found by rule and then deleted from the context (never a reply that contains an image), repaired by cutting the failed tail, or kept and masked. Replies with garbled output or calls that are malformed or write absolute paths are removed or masked in the same way. We also mask calls whose arguments violate the tool schema (a \texttt{ValidationError}) and calls that only print an exception (\texttt{KeyError}, \texttt{ImportError}, and the like; 544 replies).
  \item \emph{HexaShell-port:} 27.3\% of the candidate replies are masked at export time (6{,}792 of 24{,}853). The most frequent reasons are wrong tool usage (4{,}299), replies after the official success (1{,}700), absolute host paths (587) and thinking longer than 15K tokens (328).
  \item \emph{uni-robot-agent:} a reply is masked if its call cannot be mapped to a Pi tool or its tool result reports an error (1{,}570 replies), and the summary model's own replies are never supervised.
  \end{itemize}
\end{enumerate}
Table~\ref{tab:sft-clean} shows the effect on the number of supervised replies.

\begin{table}[!tb]
\caption{Supervised replies before and after cleaning, training sets.}
\label{tab:sft-clean}
\centering
\footnotesize
\setlength{\tabcolsep}{4pt}
\begin{tabular}{@{}lccc@{}}
\toprule
Source & Before & After & Change \\
\midrule
BEHAVIOR (Agent-dev) & 23{,}591 & 21{,}133 & $-10.4\%$ \\
RoboCasa365 (Agent-dev) & 11{,}212 & 10{,}210 & $-8.9\%$ \\
RoboDojo (Agent-dev) & 22{,}061 & 18{,}534 & $-16.0\%$ \\
RoboCasa365 (HexaShell-port) & 24{,}853 & 18{,}061 & $-27.3\%$ \\
RoboDojo (URA-port) & 12{,}908 & 7{,}674$^{\dagger}$ & $-40.5\%$ \\
\bottomrule
\end{tabular}

\raggedright\scriptsize $^{\dagger}$ Training and validation fragments together.
\end{table}

\subsection{Multi-Embodiment, Multi-Harness SFT Ablation}
\label{sec:supp-vs-base}

Table~\ref{tab:supp-sft-effect} compares the base model, a model trained on a single split, and a model trained on all splits. This is a controlled comparison, not a leaderboard: the runs use the same evaluation protocol, so the change from training is the quantity to read.

\paragraph{Setting.}
\begin{itemize}[leftmargin=*,nosep]
\item \textbf{RoboCasa365 (Panda), HexaShell port:} all tasks, two instances per task.
\item \textbf{RoboDojo (ARX X5), Agent-dev harness:} all tasks, one instance per task, \emph{without skills or task prompt}.
\end{itemize}

For RoboCasa, we arrange the layouts of each task as (1,1), (2,2), ..., (10,10) and take the first two that run stably as test instances, so RoboCasa is tested on 50 tasks $\times$ 2 instances. For RoboDojo, we use instance~0 of seed~0 as the test case, which gives 42 test runs. For the untrained RoboDojo base model, we evaluate on instance~1. The RoboDojo scores of main-text Sec.~\ref{sec:training} were measured on training instances, whereas these runs use held-out test instances only, so the RoboDojo scores differ.

\paragraph{Result.}
As shown in Tab.~\ref{tab:supp-sft-effect} and \figsupppanel{c}, full-data SFT significantly outperforms single-split SFT on the same target benchmark ($+0.14$ on RoboCasa365 and $+0.075$ on RoboDojo), confirming Prop.~\ref{prop:aux} that extra data from other embodiments and harnesses improves target-robot performance.
\begin{table}[!tb]
\caption{Effect of SFT, one block per benchmark. $\Delta$ is the gain over the row above in score points. Notice that the RoboDojo scores in Sec.~\ref{sec:training} (0.126 vs.\ 0.158) are evaluated on training instances, whereas the scores in this table (and \figsupppanel{c}) are evaluated on held-out test instances.}
\label{tab:supp-sft-effect}
\centering
\footnotesize
\setlength{\tabcolsep}{6pt}
\begin{tabular}{@{}lcc@{}}
\toprule
\multicolumn{3}{@{}c}{\textbf{RoboCasa365}} \\
\midrule
Model & Score & $\Delta$ \\
\midrule
Base & 0.15 & \\
Single-split SFT & 0.20 & \up{+0.05} \\
All-data SFT & 0.34 & \up{+0.14} \\
\midrule
\midrule
\multicolumn{3}{@{}c}{\textbf{RoboDojo}} \\
\midrule
Model & Score & $\Delta$ \\
\midrule
Base & 0.100 & \\
Single-split SFT & 0.151 & \up{+0.051} \\
All-data SFT & 0.226 & \up{+0.075} \\
\bottomrule
\end{tabular}
\end{table}

\subsection{Final Test of RoboCoder-27B}
\label{sec:supp-final}
This is the strict test behind the RoboCoder-27B bars (Ours) of \figsupppanel{a}, and it is the setting to quote for RoboCoder-27B. The other bars of that figure are the published and leaderboard references of the main text and were not rerun.

\paragraph{Setting.}
We evaluate RoboCoder-27B on BEHAVIOR, RoboDojo and RoboCasa365 under the benchmark-specific final-test settings. Summaries are written by Flash-Next. The model is trained on the full data of Sec.~\ref{sec:supp-sft-data}, and every instance used for testing was removed from the SFT data for RoboCoder-27B, so none of them is seen in training.
\paragraph{Result.}
Table~\ref{tab:res-final} summarizes the final scores (per-task details in Sec.~\ref{app:object-transfer}).
\begin{table}[h!]
\caption{Final scores (\%) of RoboCoder-27B on the three benchmarks.}
\label{tab:res-final}
\centering
\footnotesize
\setlength{\tabcolsep}{4pt}
\begin{tabular}{@{}lccc@{}}
\toprule
Model & BEHAVIOR & RoboDojo & RoboCasa365 \\
\midrule
RoboCoder-27B & 43.27 & 31.64 & 46.0 \\
\bottomrule
\end{tabular}
\end{table}

\section{3rd party Harnesses vs.\ Their Ports in \ours}
\label{sec:supp-static}

This section supplements Sec.~\ref{sec:generalization} and Tab.~\ref{tab:port} of the main text, which port third-party harnesses into the \ours workspace. Each subsection takes one native harness and follows the same three steps: how it was adapted to the workspace, the test setting, and why the score changes. All runs use Qwen3.8-Flash-Next.

\subsection{HexaShell vs.\ Its Port}
\label{sec:supp-hexashell-pair}
\label{sec:supp-hexashell}
\begin{table}[H]
\caption{Tools exported from HexaShell. The first 19 rows are the model-facing
\texttt{robot.*} calls exposed by the native RoboCasa adapter. The last three
entries are files in the port workspace rather than native \texttt{robot.*}
calls: \texttt{finish} mirrors HexaShell's native episode terminator, whereas
\texttt{feedback} and \texttt{pointcloud} are helpers added by the port.}
\label{tab:hexashell-tools}
\centering
\footnotesize
\setlength{\tabcolsep}{3pt}
\renewcommand{\arraystretch}{1.05}
\begin{tabularx}{\linewidth}{@{}>{\ttfamily\raggedright\arraybackslash}p{1.7cm}X@{}}
\toprule
\textnormal{\textbf{Tool}} & \textbf{What it does} \\
\midrule
\multicolumn{2}{@{}l}{\emph{Motion (8)}} \\
goto & move the grasp point to a position with given angles \\
arc & swing the grasp point around a line through a center \\
push\_arc & push a hinged part (lid, door edge) by up to an angle \\
rim\_pick & pick a round open dish by its near side \\
gripper & open, close, close lightly, or set a fraction in $[0,1]$ \\
home & lift if low, return the arm to its start pose \\
wait & hold still for $n\le400$ steps (0.05\,s each) \\
drive & drive the base: forward, left (m), turn (deg) \\
\midrule
\multicolumn{2}{@{}l}{\emph{Localization (4)}} \\
find & outline objects in a camera picture with SAM3 \\
measure & measure an object in 3-D from depth \\
to\_world & 3-D point at a pixel, from depth \\
to\_pixel & pixel of a 3-D point, or none if out of view \\
\midrule
\multicolumn{2}{@{}l}{\emph{State and observation (7)}} \\
task & instruction, step budget, start pose \\
state & grasp point, angles, gripper, base pose, steps left \\
image & RGB array of a camera ($512\times512\times3$) \\
depth & depth array of a camera ($512\times512$, m) \\
look & attach the pictures so the agent can see them \\
camera & intrinsics and pose of a camera \\
things & objects measured so far in the episode \\
\midrule
\multicolumn{2}{@{}l}{\emph{Additional files in \texttt{tools/}}} \\
finish & end the episode (the native HexaShell \texttt{finish} tool) \\
feedback & re-read the latest saved feedback (port helper) \\
pointcloud & batch-unproject depth pixels (port helper; no native call) \\
\bottomrule
\end{tabularx}
\end{table}

\paragraph{Adaptation and control.}
Native HexaShell~\citep{gao2026hexashell} on RoboCasa~\citep{nasiriany2026robocasa365} exposes \texttt{python} and \texttt{finish}, calling robot APIs inside a host-side read-only package outside the workspace (Sec.~\ref{sec:supp-audit}). Our port copies the supporting perception, motion, control and task-note files into \texttt{RoboHarness\_HexaShell/} and exposes the 19 model-facing \texttt{robot.*} calls (Tab.~\ref{tab:hexashell-tools}) as editable Bash CLIs in \texttt{tools/<name>.py}, plus two lightweight port helpers (\texttt{feedback.py} to re-read saved feedback and \texttt{pointcloud.py} to batch-unproject depth using the native camera convention). Under the same model-facing API, the port converts the 12-D motion command into \textbf{13 named actuator values} (7 arm torques, 1 torso torque, 3 base velocity servos, 2 gripper velocity servos) applied every 2\,ms physics substep inside the workspace.

\paragraph{Test Setting.}
We evaluate Qwen3.8-Flash-Next on all 50 tasks of HexaShell's \texttt{pretrain50} set on mobile Panda (\textbf{one paired instance per task}, identical seed, layout, style, step limit, instruction and upstream task skill) under a matched 262{,}144-token context window. Native HexaShell enforces per-instance turn caps (\texttt{max\_turns}\,$=$\,40 on 18 instances, 100 on 32) and a 7{,}200\,s wall-clock cap; the \ours port runs uncapped via Bash/Read (all 31 successes finish within 5{,}126\,s with zero compactions). Versions: HexaShell \texttt{e0190dd}, RoboCasa~\citep{nasiriany2026robocasa365} \texttt{b4684e6} (+7-file patch), robosuite \texttt{5ce6643}, MuJoCo 3.3.1.

\paragraph{Result and Turn-Limit Alignment.}
Auditing all 2{,}864 native assistant turns shows that \textbf{98.46\%} (2{,}820/2{,}864) either invoke one of the 19 \texttt{robot.*} tools (96.86\%) or run local Python/\texttt{numpy} computation (1.61\%), with zero turns spent reading docs, schemas or workspace files. Because raw Pi replies in the uncapped port (31/50\,=\,62\%) include file-based workspace/CLI overhead (\texttt{read} on docs/images, \texttt{--schema}, \texttt{ls}/\texttt{cat}/\texttt{grep}, pre-\texttt{connect()} CLI errors), truncating raw Pi turns directly (36\%) penalizes the file interface. Partitioning each port trajectory via its \texttt{tool\_call\_history} logs into (i)~\emph{native-tool turns} calling \texttt{connect()}, (ii)~\emph{local compute turns} (\texttt{numpy}/PIL without \texttt{robot.*}), and (iii)~\emph{workspace/CLI overhead} yields the aligned results in Tab.~\ref{tab:res-hexashell}: on shared successes, median native-tool turns closely match native usage (35 vs.\ 32; 1.23$\times$ vs.\ 1.60$\times$ raw). Counting native-tool + local compute turns under the native 40/100 caps yields \textbf{24/50\,=\,48\%} (identical to counting \texttt{connect()} sessions; native-tool only gives 52\%), matching or exceeding native HexaShell's 42\% across all aligned criteria while making every tool readable and editable.

\begin{table}[H]
\caption{Success rate (\%) of Qwen3.8-Flash-Next on the 50 paired RoboCasa365 instances (18 with native \texttt{max\_turns}\,$=$\,40; 32 with 100). The aligned score (48\%, bold) is reported in Tab.~\ref{tab:port}.}
\label{tab:res-hexashell}
\centering
\scriptsize
\setlength{\tabcolsep}{2.5pt}
\begin{tabularx}{\linewidth}{@{}Xrrr@{}}
\toprule
Configuration / Accounting & 40-turn & 100-turn & Total (\%) \\
\midrule
HexaShell, native (262k ctx, 40/100 cap) & 10/18 & 11/32 & 21/50 (42) \\
\midrule
\ours port (262k ctx, uncapped) & 14/18 & 17/32 & 31/50 (62) \\
\midrule
\multicolumn{4}{@{}l}{\emph{Turn-limit alignment of 262k-ctx \ours port to native 40/100 caps:}} \\
\quad All raw Pi turns $\le$ cap & 6/18 & 12/32 & 18/50 (36) \\
\quad \textbf{Native-tool + local compute $\le$ cap} & \textbf{9/18} & \textbf{15/32} & \textbf{24/50 (48)} \\
\quad Native-tool turns only $\le$ cap & 11/18 & 15/32 & 26/50 (52) \\
\quad \texttt{connect()} sessions $\le$ cap & 10/18 & 14/32 & 24/50 (48) \\
\bottomrule
\end{tabularx}
\end{table}

\subsection{URA vs.\ Its Port (RoboDojo)}
\label{sec:supp-ura}

\paragraph{Adaptation and hardware contract.}
Native uni-robot-agent (URA)~\citep{yuan2026unirobotagent} serves tools over MCP through Codex or Claude Code (Sec.~\ref{sec:supp-audit}). We port its RoboDojo Direct configuration (RGB, camera calibration, skill writing): all 11 core tools and 7 job operations (Tab.~\ref{tab:ura-tools}) become Bash CLIs under \texttt{tools/<name>.py} and \texttt{tools/ura.py} in an ARX~X5 workspace, with 10/11 identical schemas (\texttt{locate} encodes the handler's \texttt{observations} requirement in-schema) and Pi replacing Codex/MCP via native \texttt{bash} and \texttt{read}. Underneath the same tool names, an in-workspace controller solves IK (reusing URA motion profiles and RoboDojo's cuRobo config) and sends 14-value joint commands ($[$left $q_{1..6}$, grip, right $q_{1..6}$, grip$]$) with id/version checks via \texttt{env.take\_action}, keeping the agent--simulator contract joint-level.

\paragraph{Setting and result.}
We evaluate Qwen3.8-Flash-Next on dual-arm ARX~X5 across all 42 RoboDojo tasks (mean per-task score \%). The native score is reported in the original paper; the port score is the mean per-task average during data generation (Tab.~\ref{tab:res-ura}). The port reaches 30.78\%, within 0.61 points of native URA's 31.39\%.

\begin{table}[H]
\caption{Tools exported from uni-robot-agent: 11 core tools (RoboDojo Direct) and 7 job operations as Bash CLIs in \texttt{tools/}. Ticks are native RoboDojo control ticks.}
\label{tab:ura-tools}
\centering
\footnotesize
\setlength{\tabcolsep}{3pt}
\renewcommand{\arraystretch}{0.98}
\begin{tabularx}{\linewidth}{@{}>{\ttfamily\raggedright\arraybackslash}p{2.5cm}X@{}}
\toprule
\textnormal{\textbf{Tool}} & \textbf{What it does} \\
\midrule
\multicolumn{2}{@{}l}{\emph{Observation (3)}} \\
capture\_scene & snapshot of RGB, calibration and state; 0 ticks \\
view\_env\_state & read latest or initial past observation; 0 ticks \\
locate & triangulate 3-D point from 2--3 cameras; 0 ticks \\
\midrule
\multicolumn{2}{@{}l}{\emph{Motion (4)}} \\
move & go to grasp-point pose per arm with IK; 1--24 ticks \\
grasp & close gripper at measured pose; 1--12 ticks \\
release & open gripper; 1--12 ticks \\
hold & advance physics with current targets; 1--40 ticks \\
\midrule
\multicolumn{2}{@{}l}{\emph{Skills and session (4)}} \\
list\_skills & list skill Markdown files; 0 ticks \\
read\_skill & read one skill file; 0 ticks \\
write\_skill & write skill file under \texttt{skills/robodojo/}; 0 ticks \\
finish & declare completion with summary \\
\midrule
\multicolumn{2}{@{}l}{\emph{Job operations (7), in \texttt{tools/ura.py}}} \\
get\_tool\_result & result of a submitted job \\
wait\_for\_tool(s) & wait for one or several jobs \\
list\_tool\_jobs & list jobs \\
cancel\_tool & cancel a job that has not started \\
queue\_status & state of the queue \\
flush\_queue & clear actions that have not started \\
\bottomrule
\end{tabularx}
\end{table}

\begin{table}[H]
\caption{Score (\%) of Qwen3.8-Flash-Next under native URA and the RoboHarness port (main-text Tab.~\ref{tab:port}).}
\label{tab:res-ura}
\centering
\footnotesize
\begin{tabular}{@{}lcc@{}}
\toprule
Model & URA & RoboHarness port \\
\midrule
Qwen3.8-Flash-Next & 31.39$^{\dagger}$ & 30.78 \\
\bottomrule
\end{tabular}

\raggedright\scriptsize $^{\dagger}$ Reported in the original paper.
\end{table}

\clearpage
\section{Built-In Toolset in RoboHarness (BEHAVIOR)}
\label{sec:supp-behavior}


\subsection{Keypoint-Constraint Action Toolset and Spatial Grounding}
\label{sec:supp-v1-method}

Lessons from coding agents show that a minimal set of simple, orthogonal primitive tools outperforms a collection of predefined features and heavy scaffolding. Inspired by this principle, we argue that an LLM's potential on embodied tasks is best unlocked by giving it the minimal primitive action tools rather than layers of engineered features.

We construct primitive action tools from spatial relations between the gripper and objects. Section~\ref{sec:proof-completeness} establishes that geometric constraints can represent (i)~every feasible grasp pose with a non-collinear gripper reference triangle and (ii)~every prescribed relative rigid pose after a rigid grasp.

With the primitive action tools in place, the remaining challenge is spatial understanding: although VLMs cannot reliably reason about 3D spatial relations, their grounding training enables them to localize specific points on objects with reasonable accuracy. We therefore use spatial grounding as the bridge between vision-language models and robots. The LLM marks points of interest on a 2D image and continuously receives their positions via optical-flow tracking and depth back-projection; with this information, it composes accurate primitive actions to perform complex manipulation.

\subsubsection{Geometry State for Spatial Grounding}
\label{sec:geometry-state}

To make primitive action tools more usable to the LLM, we design a persistent geometry-state tracking mechanism that continuously returns the coordinates of objects of interest to the LLM, enhancing its accurate spatial perception capabilities.

\paragraph{Robot coordinate definition.}
\label{sec:body-frame}
We first define a unified camera-frame coordinate system for diverse robot configurations, making it straightforward for the LLM to read spatial relations: the $x$ axis points forward from the camera view, the $y$ axis left, and the $z$ axis upward (Fig.~\ref{fig:robot-frames}).

\begin{figure}[t]
  \centering
  \includegraphics[width=0.8\linewidth]{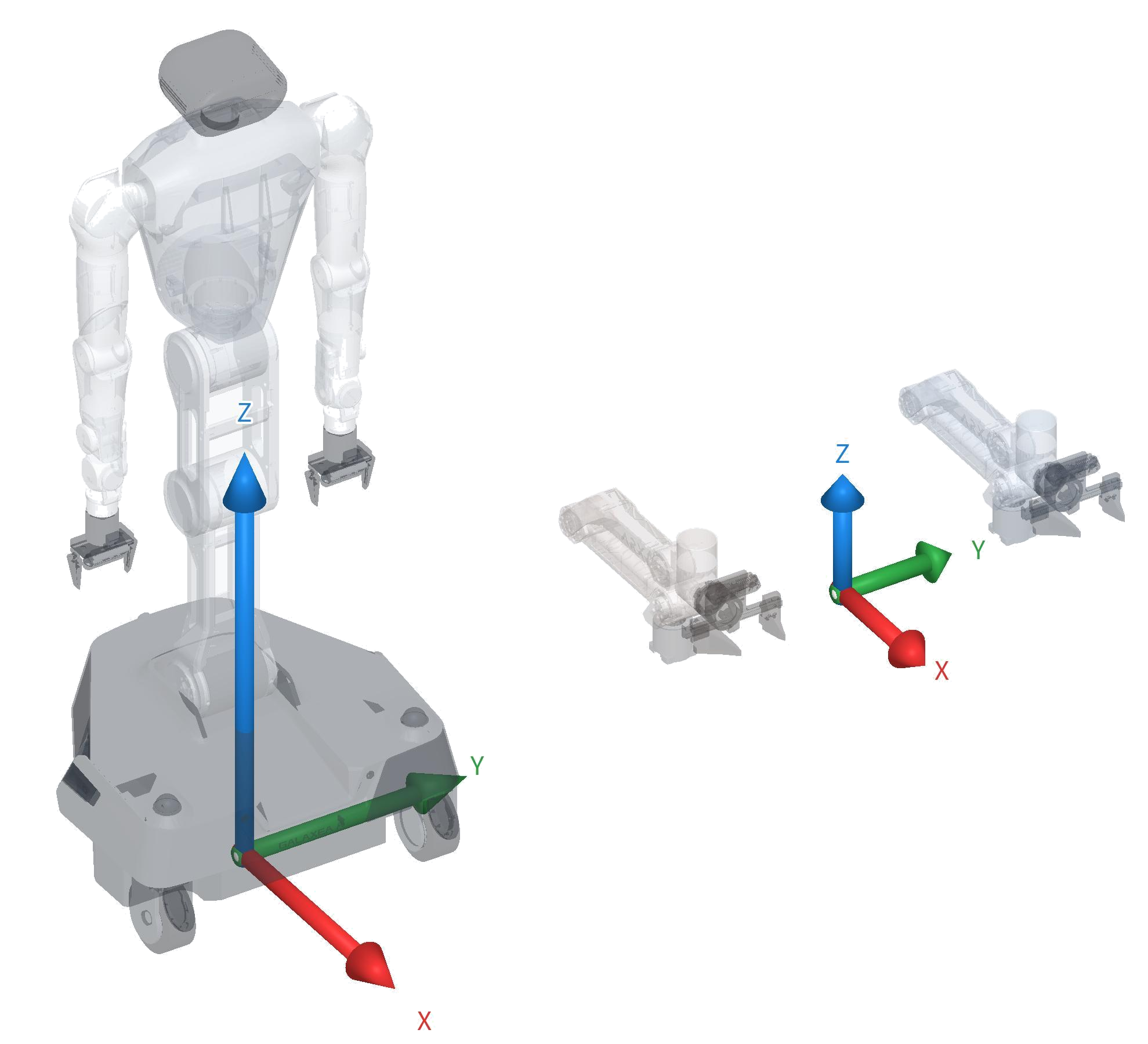}
  \caption{Unified coordinate on two embodiments: the R1~Pro in BEHAVIOR \citep{li2024behavior} (left) and the ARX~X5 in RoboDojo \citep{chen2026robodojo} (right).}
  \label{fig:robot-frames}
\end{figure}

\paragraph{Track keypoints.}
\label{sec:track-keypoints}
With a unified coordinate frame in place, we design the \texttt{track\_\allowbreak{}object} tool for the LLM to retrieve coordinates of points in which it is interested using no simulation priors. The procedure (Fig.~\ref{fig:track-tool}) proceeds as follows:
\begin{enumerate}[leftmargin=*,nosep]
\item The LLM selects one or more points of interest that are visible in the head camera.
\item The harness continuously tracks each marked point across successive head-camera frames.
\item The latest 2D location is projected onto the depth view to recover the metric distance.
\item The pixel coordinate and depth are then back-projected into the robot camera frame.
\end{enumerate}

\begin{figure}[t]
  \centering
  \includegraphics[width=\linewidth]{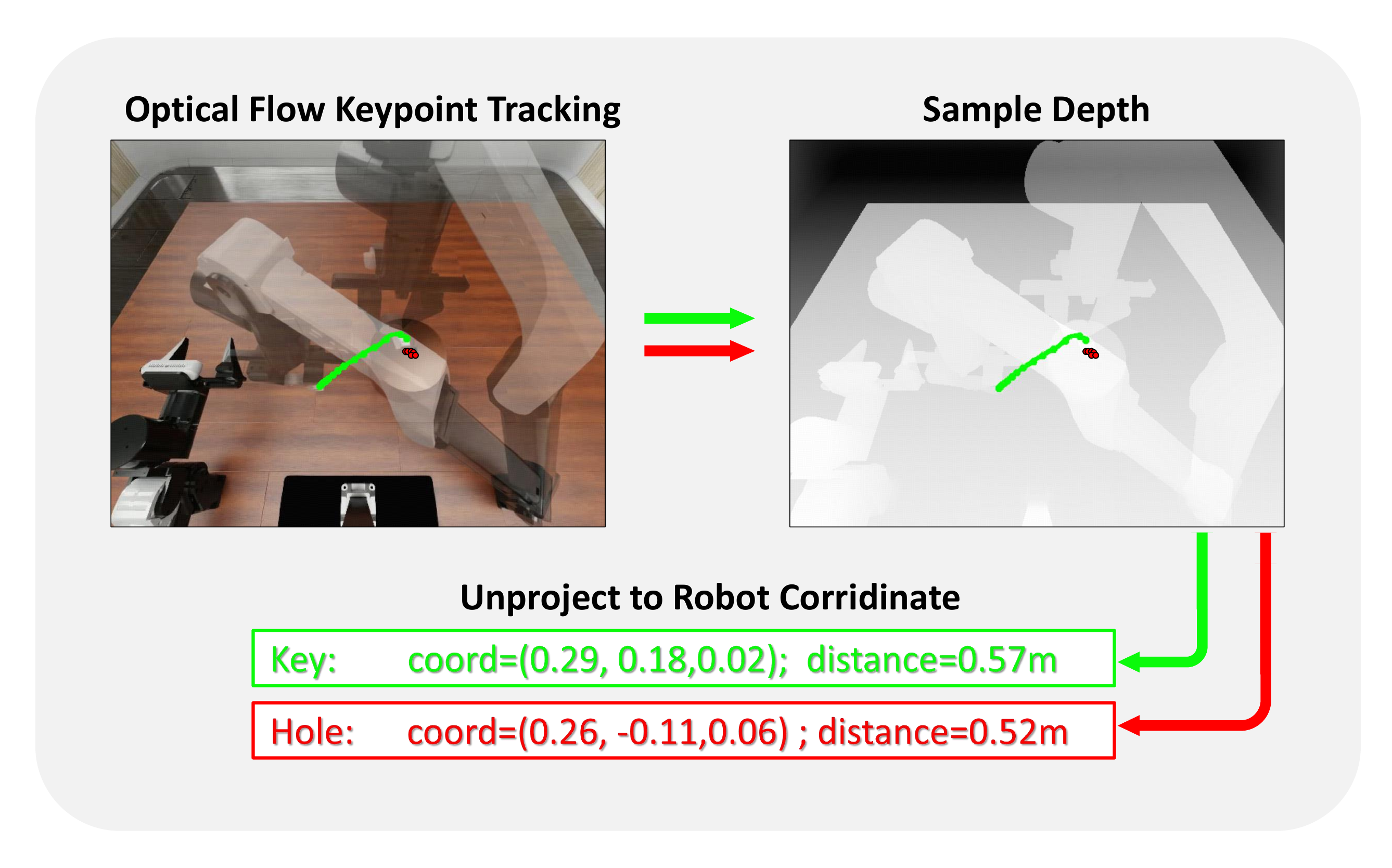}
  \caption{Pipeline of the \texttt{track\_\allowbreak{}object} tool: optical-flow keypoint tracking followed by depth back-projection into the robot coordinate frame.}
  \label{fig:track-tool}
\end{figure}

\subsubsection{Action Toolset}
\label{sec:action-toolset}

Our goal is to design a minimal action toolset. Therefore, we start from object-interaction relations and build the toolset bottom-up.

A manipulation task can be viewed as a two-stage process: first the gripper establishes a spatial relationship with an object, and then an end-effector action drives this object to satisfy a geometric relationship with another object. For instance, in an insert-key task, the gripper first grasps the key, then moves it until it satisfies the insertion constraint with the lock. Consequently, our action toolset only needs to capture these two stages.

We implement the first stage with keypoint-and-constraint tools: the LLM selects points on the image and specifies the geometric constraints they must satisfy; the harness automatically solves and executes the corresponding actions. To support model fine-tuning, we additionally provide adjustment tools that directly command the robot's low-level degrees of freedom.

Section~\ref{sec:proof-completeness} shows that every feasible grasp pose with non-collinear gripper reference points can be represented by two fingertip coincidences and one directed collinearity constraint. Once the grasp is rigid, any desired relative pose of the held object can be represented by two non-parallel directed line correspondences. These results establish completeness of the pose representation; execution additionally requires feasible inverse kinematics and a collision-free path.

\paragraph{Keypoint-constraint tools for manipulation.}
\label{sec:keypoint-tools}
The model supplies a set of input points together with their spatial relations, and the harness automatically solves and executes the corresponding action. The model must distinguish between on-hand and off-hand points: on-hand points lie on the gripper or on a held object and their positions can be modified, whereas off-hand points are fixed in the scene. The harness drives the on-hand points so that they satisfy the spatial constraints specified by the model through the following pipeline:
\begin{enumerate}[leftmargin=*,nosep]
    \item Sample action candidates and select those satisfying the constraints.
    \item Discard candidates whose inverse kinematics are infeasible or unreachable.
    \item Plan a motion trajectory from the current pose to the target pose.
    \item Execute the trajectory.
\end{enumerate}
In principle, the model can express any constraint through coordinate relations among the input points. However, expressions such as collinearity and coplanarity are both verbose to write and frequently required, so we encapsulate the most common patterns as quick requests; their specific modes are listed in Table~\ref{tab:quick-constraints}.

\begin{table}[t]
\caption{Modes for the keypoint-constraints tool. The \emph{coordinate} mode takes arbitrary input points and moves them to satisfy the given coordinate constraints; the rest are quick requests.}
\label{tab:quick-constraints}
\centering
\footnotesize
\setlength{\tabcolsep}{3pt}
\renewcommand{\arraystretch}{1.1}
\begin{tabularx}{\linewidth}{@{}>{\raggedright\arraybackslash}p{1.9cm}cc>{\raggedright\arraybackslash}X@{}}
\toprule
\textbf{Quick request} & \textbf{On} & \textbf{Off} & \textbf{Description} \\
\midrule
touch & $1$ & $1$ & on-hand point touch an off-hand point \\
linear & $2$ & $1$ or $2$ & two on-hand points collinear with off-hand points \\
line-vertical-plane & $2$ & $3$ & on-hand line orthogonal to an off-hand plane \\
plane-parallel & $3$ & $3$ & on-hand plane parallel certain plane \\
line-vertical-ground & $2$ & $0$ & on-hand line vertical to the ground \\
face to & $3$ & $0$ & face the target surface \\
move to reach & $0$ & $1$ & move chassis to make target reachable \\
grasp & $0$ & $1$ & grasp an object \\
\bottomrule
\end{tabularx}
\end{table}

\paragraph{Adjustment tools.}
\label{sec:adjustment-tools}
To refine robot action or to improve robustness when the constraints fail (for example, when the object slips out of the gripper), we also expose direct low-level robot control interfaces. These interfaces share the same coordinate frame defined above, so that the LLM can interpret and use them consistently (Table~\ref{tab:adjustment-tools}).

\begin{table}[t]
\caption{Adjustment tools. Displacements are incremental; chassis \texttt{spin} is a yaw increment.}
\label{tab:adjustment-tools}
\centering
\footnotesize
\setlength{\tabcolsep}{2.5pt}
\begin{tabularx}{\linewidth}{@{}>{\raggedright\arraybackslash}p{1.75cm}>{\centering\arraybackslash}p{2.45cm}X@{}}
\toprule
Name & Parameters & Meaning \\
\midrule
adjust left eef & $(\Delta x,\Delta y,\Delta z,r,p,y)$ & adjust the left end-effector \\
adjust right eef & $(\Delta x,\Delta y,\Delta z,r,p,y)$ & adjust the right end-effector \\
adjust chassis & (fwd, side, spin) & move the chassis \\
adjust height & $\Delta h$ & raise or lower the body \\
adjust pitch & $\Delta\theta$ & change the body pitch \\
\bottomrule
\end{tabularx}
\end{table}

\subsection{BEHAVIOR Five-Instance Scores}
\label{app:behavior-instances}

For Astra and RoboCoder-27B, each BEHAVIOR task is tested on instance~301. Qwen3.8-Flash-Next with skills is evaluated on the first ten tasks with five instances per task, for $10\times5=50$ episodes; its mean score is 44.54\% (Table~\ref{tab:behavior-astra}). The robot is the R1~Pro embodiment, and all scores use the official BEHAVIOR evaluator.

An episode ends when the task goal is satisfied or when the simulator reaches the task's official maximum tick count. Table~\ref{tab:behavior-max-ticks} lists those limits.

\begin{table}[t]
\caption{BEHAVIOR Challenge scores under the test settings listed below (GPT-6 Astra on BEHAVIOR Challenge 2026; Qwen3.8-Flash-Next on BEHAVIOR Challenge 2025).}
\label{tab:behavior-astra}
\centering
\small
\begin{tabularx}{\linewidth}{@{}Xr@{}}
\toprule
Setting & Mean $q$-score \\
\midrule
Astra, no skills, 100 tasks $\times$ 1 instance & 45.7\% \\
Qwen3.8-Flash-Next w/ skills, 10 tasks $\times$ 5 instances & 44.54\% \\
\bottomrule
\end{tabularx}
\end{table}

\begin{table}[t]
\caption{Official maximum tick count for each BEHAVIOR Challenge 2025 task in this sample. The episode stops at this count unless the goal is already satisfied.}
\label{tab:behavior-max-ticks}
\centering
\footnotesize
\setlength{\tabcolsep}{4pt}
\begin{tabular}{@{}rlr@{}}
\toprule
ID & Task & Max ticks \\
\midrule
0 & turning on radio & 4299 \\
1 & picking up trash & 10535 \\
2 & putting away Halloween decorations & 27664 \\
3 & cleaning up plates and food & 27392 \\
4 & can meat & 23694 \\
5 & setting mousetraps & 20343 \\
6 & hiding Easter eggs & 15239 \\
7 & picking up toys & 37781 \\
8 & rearranging kitchen furniture & 17886 \\
9 & putting up Christmas decorations inside & 27437 \\
\bottomrule
\end{tabular}
\end{table}

\subsection{Proof: Completeness of the Keypoint-Constraint Representation}
\label{sec:proof-completeness}

The following results concern representation of grasp poses and rigid relative poses. The references may be coordinate points, as allowed by the coordinate mode of Table~\ref{tab:quick-constraints}.

\paragraph{Rigid-pose notation.}
Work in Euclidean space $\mathbb{R}^3$, with norm $\|\cdot\|$, inner product $\cdot$ and cross product $\times$. Let $I_3$ be the $3\times3$ identity matrix and define
\[
SO(3):=\{Q\in\mathbb{R}^{3\times3}:Q^\top Q=I_3,\ \det Q=1\}.
\]
The set $SE(3)$ consists of maps $x\mapsto Qx+b$, where $x,b\in\mathbb{R}^3$ and $Q\in SO(3)$. Thus $Q$ is a proper rotation and $b$ a translation. Such rotations preserve cross products: $Q(u\times v)=(Qu)\times(Qv)$ for vectors $u,v\in\mathbb{R}^3$.

\subsubsection{Any Grasp Pose Can Be Expressed by Keypoint Constraints}
\label{sec:grasp-pose}

Fix a gripper opening. In its local frame, let $F_L,F_R,G\in\mathbb{R}^3$ be the left fingertip, right fingertip and gripper center. Define $c_g:=(F_L+F_R)/2$, $u:=F_R-F_L$, $w:=G-c_g$ and $\rho:=\|w\|$. Assume $u\ne0$, $w\cdot u=0$ and $\rho>0$; hence the three points are non-collinear. Let $p\in SE(3)$ be the unknown map from gripper coordinates to scene coordinates and $p_\star\in SE(3)$ a desired feasible pose. Set $f_l:=p_\star(F_L)$, $f_r:=p_\star(F_R)$ and $c:=(f_l+f_r)/2$. For any fixed scalar $s>0$, choose references $a_1:=f_l$, $a_2:=f_r$ and $a_3:=c+s(p_\star(G)-c)$. The reference $a_3$ need not lie on an object surface.

\begin{figure}[t]
\centering
\includegraphics[width=0.8\linewidth]{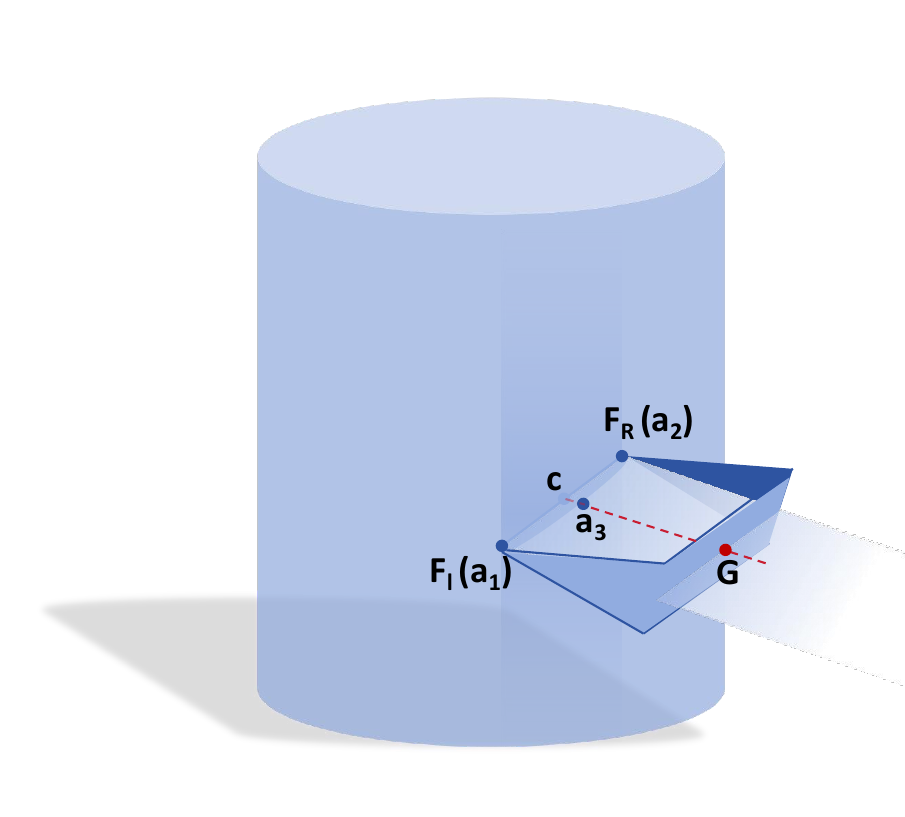}
\caption{The left fingertip $F_L$, right fingertip $F_R$ and gripper center $G$ form a non-collinear reference triangle at a fixed opening.}
\label{fig:gripper-triangle}
\end{figure}

\begin{proposition}[Gripper keypoints uniquely determine a grasp pose]\label{prop:grasp-keypoints}
Under the above assumptions, the following constraints on $p$, with an unknown scalar $\lambda>0$, have the unique pose solution $p=p_\star$:
\begin{equation}
\begin{gathered}
p(F_L)=a_1,\qquad p(F_R)=a_2,\\
p(G)-c=\lambda(a_3-c),\qquad\lambda>0.
\end{gathered}
\end{equation}
\end{proposition}
\begin{proof}
The desired pose satisfies the constraints with $\lambda=1/s$, proving existence. For any solution $p$, its first two constraints give $p(c_g)=c$, since rigid maps preserve midpoints. Distance preservation then gives $\|p(G)-c\|=\|G-c_g\|=\rho$. Likewise $p_\star(c_g)=c$, so $\|a_3-c\|=s\rho>0$. Taking norms in the last constraint gives $\rho=\lambda s\rho$, hence $\lambda=1/s$ and $p(G)=p_\star(G)$. Thus $p$ and $p_\star$ agree at all three gripper points.
To prove uniqueness explicitly, write $S:=p_\star^{-1}\circ p$, $S(x)=Qx+b$ with $Q\in SO(3)$ and $b\in\mathbb{R}^3$. This map fixes $F_L,F_R,G$ and their midpoint $c_g$, so subtracting fixed-point equations gives $Qu=u$ and $Qw=w$. Cross-product preservation gives $Q(u\times w)=u\times w$. The three vectors $u,w,u\times w$ form a basis because $u,w$ are nonzero and perpendicular. Hence $Q=I_3$; the equation $S(F_L)=F_L$ then gives $b=0$. Therefore $S$ is the identity and $p=p_\star$.
\end{proof}

\subsubsection{Any Relative Rigid Pose Can Be Expressed by Keypoint Constraints}
\label{sec:in-hand-motion}

Let $A,B$ denote two rigid objects and let $T_{\mathrm{true}}\in SE(3)$ be a prescribed map from $B$'s coordinates to $A$'s. Choose points $x_1,x_2\in\mathbb{R}^3$ and unit vectors $v_1,v_2\in\mathbb{R}^3$ with $v_1\times v_2\ne0$. For $i=1,2$, define the directed line $L_{Bi}:=\{x_i+\eta v_i:\eta\in\mathbb{R}\}$ with direction $v_i$, and set $L_{Ai}:=T_{\mathrm{true}}(L_{Bi})$, including its transformed direction. Equality of directed lines requires both equal supporting lines and equal unit directions. Let $T\in SE(3)$ be an unknown relative pose.

\begin{proposition}[Directed line constraints fix the relative pose]\label{prop:two-line-se3}
For every prescribed $T_{\mathrm{true}}$, the two correspondences $T(L_{Bi})=L_{Ai}$, $i=1,2$, have the unique solution $T=T_{\mathrm{true}}$.
\end{proposition}

\begin{figure}[t]
\centering
\includegraphics[width=\linewidth]{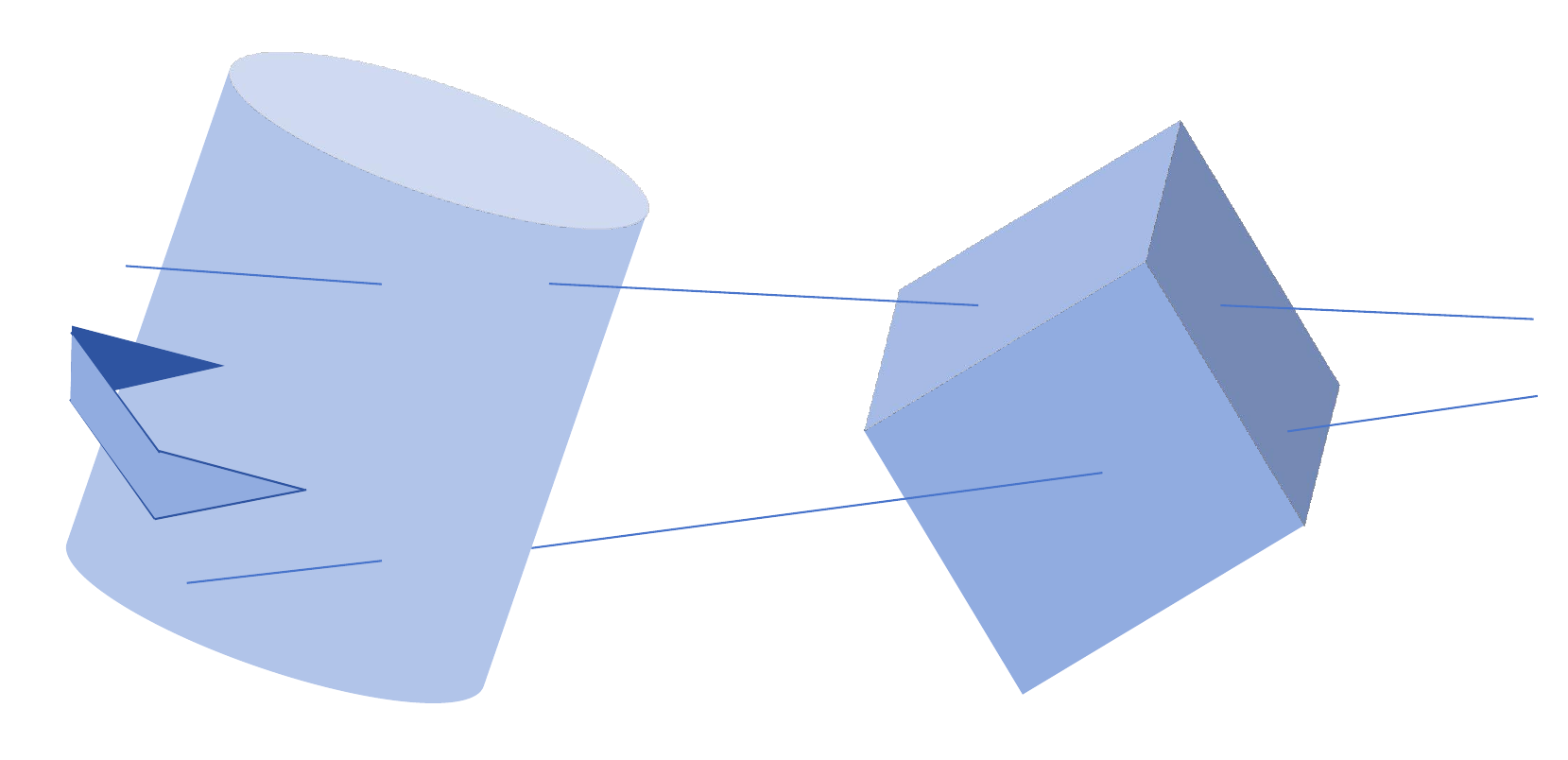}
\caption{Two non-parallel directed line correspondences determine a unique relative pose. Ordered reference points fix the line directions.}
\label{fig:gripper-pose}
\end{figure}

\begin{proof}
By definition, $T_{\mathrm{true}}$ satisfies both correspondences. For any solution $T$, set $S:=T_{\mathrm{true}}^{-1}\circ T$ and write $S(x)=Qx+b$, with $Q\in SO(3)$ and $b\in\mathbb{R}^3$. Both correspondences imply $S(L_{Bi})=L_{Bi}$ as directed lines. Equality of unit directions gives $Qv_i=v_i$. Therefore $Q(v_1\times v_2)=v_1\times v_2$; the non-parallelism makes $v_1,v_2,v_1\times v_2$ a basis, so $Q=I_3$. Now $S$ is the translation $x\mapsto x+b$. Since $x_i+b\in L_{Bi}$, there is a real scalar $\beta_i$ such that $b=\beta_i v_i$, for each $i$. A nonzero $b$ would make $v_1$ and $v_2$ parallel, contradicting their choice. Thus $b=0$, $S$ is the identity and $T=T_{\mathrm{true}}$.
\end{proof}

\section{Harness--Model Co-Training: What Changed at Each Step}
\label{sec:supp-cotrain}

This section supplements Fig.~\ref{fig:cotrain-curve} of the main text (RoboDojo ARX~X5, Agent-dev harness, 42 tasks). To track in-loop co-training progression, evaluation uses instance~1 across seeds 0, 1, and 2 (part of the training pool). Each score averages over the available seeds per task and then across the 42 tasks: the first three checkpoints (0.1071, 0.1583, 0.1778) are evaluated over paired multi-seed runs, and the last three (0.1781, 0.1940, 0.2012) on seed~0. Table~\ref{tab:cotrain-steps} lists what changed in the harness at each step.

\begin{table}[H]
\caption{Harness changes at each step of Fig.~\ref{fig:cotrain-curve}. Orange steps edit the harness, navy steps train the model; the rightmost column lists only what differs between the two endpoints.}
\label{tab:cotrain-steps}
\centering
\footnotesize
\setlength{\tabcolsep}{3pt}
\renewcommand{\arraystretch}{1.08}
\begin{tabularx}{\linewidth}{@{}>{\raggedright\arraybackslash}p{1.45cm}>{\raggedright\arraybackslash}X@{}}
\toprule
Step & Harness changes \\
\midrule
0 (orange)\newline $0\to0.1071$ & Initialize ARX~X5 workspace after inspecting existing harnesses: 4 skills, 111 files (105 in \texttt{tools/}); one multimodal demonstration (obtained from video) for each task. \\
\midrule
0$\to$1 (navy)\newline $\to0.1583$ & None in prompts, skills or hints. Three of 56 pairs differ in the \texttt{bash} timeout text and a few files (\texttt{scheduler.py}, \texttt{control\_clock.py}, \texttt{runtime.py}). \\
\midrule
1 (orange)\newline $\to0.1778$ & Hint removed in 42 tasks. \\
\midrule
1$\to$2 (navy)\newline $\to0.1781$ & New skill \texttt{moving-objects} and 8 robot macros added, system and task prompts modified in 42 tasks. \\
\midrule
2 (orange)\newline $\to0.1940$ & Hint added and system prompt modified in 1 task (\texttt{cover\_blocks}). \\
\midrule
2$\to$3 (navy)\newline $\to0.2012$ & Hint added in 35 tasks, task prompt modified in 41 and system prompt in 34. \\
\bottomrule
\end{tabularx}
\end{table}

\section{100-Object Picking Test}
\label{app:robot-hours}

\paragraph{Test Settings.}
The evaluation uses the same catalog of $100$ object assets. A trial runs for at most $2000$ simulation steps and succeeds when the target object's root height satisfies $\max_t z_t-z_0\ge5$\,cm, where $z_0$ is its initial height.
The LIBERO-trained policies are evaluated with the LIBERO robot in a cleared tabletop scene \citep{liu2023libero}; each object is placed at the center of the table, with the instruction \texttt{pick up the \{name\}}.
ASPIRE and \ours use the R1~Pro with the object on the floor and the robot initialized in a stooping pose. ASPIRE is the single-call pick-function baseline described below.
These settings share object assets and a lift criterion but differ in robot embodiment, scene and initial posture; the reported results are not a controlled comparison on a common robot and scene. Equal simulation-step limits also need not imply equal wall-clock budgets across environments.

\paragraph{Robot-data accounting.}
An hour counts robot recordings and human egocentric manipulation video once. Web-video generation and VQA are not counted.

\paragraph{Qwen3.8-27B.}
The run calls the served model \texttt{Qwen3.8-27B-FP8} with the skill prompt. This model has no released robot checkpoint on Hugging Face. The trial uses the shared $2000$-step protocol and the $5$\,cm lift. The run has no robot manipulation training set, so the robot-data hours are $0$.

\paragraph{Qwen3.8-Flash.}
The run calls the served model \texttt{Qwen3.8-Flash-Next-FP8} with the skill prompt. This model has no released robot checkpoint on Hugging Face. The run has no robot manipulation training set, so the robot-data hours are $0$.

\paragraph{GPT-6-luna.}
The run calls the served model \texttt{gpt-6-luna} with the skill prompt. This model has no released robot checkpoint on Hugging Face. The run has no robot manipulation training set, so the robot-data hours are $0$.

\paragraph{$\pi_{0.5}$.}
The run uses the released \texttt{pi05\_\allowbreak{}libero} checkpoint from the openpi release. The policy replans every $5$ actions. The paper states about $400$ hours of mobile-manipulator data, and states that $97.6\%$ of the examples in the first training phase come from other robots or the web \citep{intelligence2025pi05}. Converting the whole first-phase mixture by that share gives $16{,}666.7$ hours. The paper does not separate the web portion from the other robot data.

\paragraph{Faster-WAM.}
The run uses \texttt{hustvl/\allowbreak{}FasterWAM}, the released LIBERO weights at step $21700$. The policy replans every $10$ actions. The paper initializes the model from Wan2.2-5B, marks embodied pretraining as absent, and trains on four LIBERO suites of $500$ demonstrations for $10$ epochs \citep{zhao2026fasterwam,wan2025wan22}. The training release \texttt{yuanty/\allowbreak{}LIBERO-fastwam} states $1{,}712$ episodes and $277{,}713$ frames at $20$ frames per second, which converts to $3.86$ hours \citep{libero2026fastwamdata}. The Wan pretraining corpus is not counted \citep{wan2025wan}.

\paragraph{Galaxea G0.5.}
The run uses the \texttt{g05-libero} checkpoint released in \texttt{OpenGalaxea/\allowbreak{}G05}. The policy predicts $10$ action steps at a time. The G0.5 paper states $14$ embodiments and does not state a training duration \citep{liu2026g05}. The estimate therefore uses stage 1 of the previous G0 paper, which states about $1{,}000$ hours of Open X-Embodiment, $500$ hours of Galaxea Open-World, and $200$ hours collected in-house \citep{jiang2025galaxea,oxe2024}. The no-op-filtered LIBERO demonstrations used for this checkpoint convert to $3.80$ hours at $20$\,Hz. Adding these durations gives $1{,}703.8$ hours.

\paragraph{OpenWAM-$\alpha$.}
The run uses \texttt{OpenWAM/\allowbreak{}OpenWAM-Alpha-Sim-LIBERO} at step $10690$. Denoising uses $10$ steps, and the inference horizon is $10$ actions. The paper lists $1{,}442$ hours of egocentric human video, $1{,}794$ hours of AgiBot, $686$ hours of RoboCOIN, $1{,}007$ hours of DROID, and $1{,}440$ hours of InternData-A1, and states that this mixture is seen for one epoch \citep{wang2026openwam,agibot2025world,wu2025robocoin,khazatsky2024droid,tian2025interndata}. Adding these durations gives $6{,}369$ hours. The Wan2.2 initialization is not counted \citep{wan2025wan,wan2025wan22}.

\paragraph{LingBot-VA.}
The run uses \texttt{robbyant/\allowbreak{}lingbot-va-posttrain-libero-long}. The trial uses the shared setting of $2000$ steps and a $5$\,cm lift. The paper states approximately $16$K hours of robot manipulation data, and states that pretraining consumes $1.4$T tokens \citep{lingbotva2026}. The plotted value is the stated $16{,}000$ hours. The token count records how much data the run consumes, and it is not added again as a second copy of the corpus.

\paragraph{OpenVLA-OFT.}
The run uses \texttt{moojink/\allowbreak{}openvla-7b-oft-finetuned-libero-10}. The unnormalization key is \texttt{libero\_\allowbreak{}10\_\allowbreak{}no\_\allowbreak{}noops}, and the client executes the full $8$-step action chunk. The OpenVLA paper states $970$k demonstrations from $27$ Open X-Embodiment datasets, trained for $27$ epochs, with the DROID portion removed for the last third of training \citep{kim2024openvla,oxe2024}. The released frame counts of those $27$ datasets convert to $1{,}651.4$ hours. The OFT paper fine-tunes this checkpoint on LIBERO and does not add another pretraining corpus \citep{kim2025oft}.

\paragraph{ABot-M0.}
The run uses \texttt{acvlab/\allowbreak{}ABot-M0-LIBERO} at step $40000$. The action chunk is $10$ steps, and the unnormalization key is \texttt{franka}. The paper states that UniACT contains more than six million trajectories and $9{,}500$ hours \citep{abotm02026}. The plotted value is the stated $9{,}500$ hours.

\paragraph{Xiaomi-Robotics-0.}
The run uses \texttt{XiaomiRobotics/\allowbreak{}Xiaomi-Robotics-0-LIBERO}. The task id is \texttt{libero\_\allowbreak{}all}, and the policy replans every $10$ actions. The paper states about $200\times 10^{6}$ robot timesteps on a unified $30$\,Hz timeline \citep{xiaomi2026robotics0}. These timesteps convert to $1{,}851.9$ hours. The $338$ hours and $400$ hours of in-house tasks stated in the paper are already inside this total.

\paragraph{X-VLA.}
The run uses \texttt{2toINF/\allowbreak{}X-VLA-Libero}. The domain id is $3$, and denoising uses $10$ steps. The pretraining table lists episode counts: AgiBot has $141$K episodes, DROID is listed as two rows of $45$K episodes, and RoboMIND is listed as $19$K, $2$K, $25$K, and $11$K episodes \citep{zheng2025xvla}. AgiBot states $1{,}001{,}552$ trajectories and $2{,}976.4$ hours, which converts to $419.0$ hours for the listed $141$K episodes \citep{agibot2025world}. DROID states $76$k successful trajectories and $350$ hours, and the two $45$K rows are one subset of that release, which converts to $207.2$ hours \citep{khazatsky2024droid}. RoboMIND states $107$k trajectories and $305.5$ hours for the full corpus, which converts to $162.7$ hours for the listed $57$K episodes \citep{wu2024robomind}. Adding these three durations gives $788.9$ hours. The sampling weights in the X-VLA table are mixture weights, and they are not used as durations.

\paragraph{ASPIRE pick function.}
The baseline has no learned checkpoint. Each object is scored by one call of the pick function built from the R1~Pro examples in the released ASPIRE code (\texttt{examples/behavior}, release v1.0, commit \texttt{fa00f0e}) \citep{lu2026aspire}, and the same $5$\,cm lift is a success. The baseline has no robot manipulation training set, so the robot-data hours are $0$.

\onecolumn
\section{Detailed Score}
\label{app:object-transfer}

This section gives the per-task score details for the three benchmark evaluations (Tabs.~\ref{tab:behavior-details}--\ref{tab:robocasa-details}), followed by the per-object outcomes (Tab.~\ref{tab:object-outcomes}) used in Fig.~\ref{fig:robot-hours}. Scores are reported on the benchmark-native $[0,1]$ scale unless stated otherwise.

\newlength{\suppdetailwidth}
\setlength{\suppdetailwidth}{324pt}

\subsection{BEHAVIOR}
\begingroup
\scriptsize
\setlength{\tabcolsep}{0pt}
\renewcommand{\arraystretch}{0.94}
\setlength{\LTcapwidth}{\suppdetailwidth}
\setlength{\LTleft}{84.5pt}
\setlength{\LTright}{84.5pt}
\begin{longtable}{@{}>{\raggedleft\arraybackslash}p{14pt}@{\hspace{3pt}}>{\raggedright\arraybackslash}p{158pt}*{5}{>{\raggedleft\arraybackslash}p{25pt}}>{\raggedleft\arraybackslash}p{27pt}@{}}
\caption{Per-task BEHAVIOR scores for RoboCoder-27B. Scores are in $[0,1]$; columns 1--5 are the five evaluation attempts.}\label{tab:behavior-details}\\
\toprule
ID & Task & 1 & 2 & 3 & 4 & 5 & Mean \\
\midrule
\endfirsthead
\toprule
ID & Task & 1 & 2 & 3 & 4 & 5 & Mean \\
\midrule
\endhead
\bottomrule
\endfoot
0 & turning on radio & 1.0000 & 0.0000 & 1.0000 & 1.0000 & 0.0000 & 0.6000 \\
1 & picking up trash & 1.0000 & 1.0000 & 1.0000 & 1.0000 & 1.0000 & 1.0000 \\
2 & putting away Halloween decorations & 0.2857 & 0.2857 & 0.7143 & 1.0000 & 0.2857 & 0.5143 \\
3 & cleaning up plates and food & 0.2857 & 0.1429 & 0.2857 & 0.2857 & 0.2857 & 0.2571 \\
4 & can meat & 0.0000 & 0.0000 & 0.0000 & 0.0000 & 0.0000 & 0.0000 \\
5 & setting mousetraps & 0.6667 & 0.6667 & 0.8333 & 0.8333 & 0.8333 & 0.7667 \\
6 & hiding Easter eggs & 1.0000 & 0.4444 & 1.0000 & 0.0000 & 0.6667 & 0.6222 \\
7 & picking up toys & 0.3333 & 0.5000 & 0.1667 & 0.1667 & 0.5000 & 0.3333 \\
8 & rearranging kitchen furniture & 0.2500 & 0.0000 & 0.2500 & 0.0000 & 0.0000 & 0.1000 \\
9 & putting up Christmas decorations inside & 0.0000 & 0.2222 & 0.1111 & 0.2222 & 0.1111 & 0.1333 \\
\midrule
\multicolumn{2}{@{}l}{\textbf{Task 0--9 mean}} & \multicolumn{5}{c}{} & \textbf{0.43269} \\
\end{longtable}
\endgroup

\subsection{RoboDojo}
\begingroup
\scriptsize
\setlength{\tabcolsep}{0pt}
\renewcommand{\arraystretch}{0.91}
\setlength{\LTcapwidth}{\suppdetailwidth}
\setlength{\LTleft}{84.5pt}
\setlength{\LTright}{84.5pt}
\begin{longtable}{@{}>{\raggedleft\arraybackslash}p{16pt}@{\hspace{3pt}}>{\raggedright\arraybackslash}p{108pt}*{3}{>{\raggedleft\arraybackslash}p{38pt}}>{\raggedleft\arraybackslash}p{38pt}>{\raggedleft\arraybackslash}p{48pt}@{}}
\caption{Per-task RoboDojo scores for RoboCoder-27B across three seeds (scores in $[0,1]$). }\label{tab:robodojo-details}\\
\toprule
ID & Task & Seed 0 & Seed 1 & Seed 2 & Mean & Dim. mean \\
\midrule
\endfirsthead
\toprule
ID & Task & Seed 0 & Seed 1 & Seed 2 & Mean & Dim. mean \\
\midrule
\endhead
\bottomrule
\endfoot
\multicolumn{7}{@{}l}{\textbf{Generalization (12 tasks)}} \\
3 & stack blocks & 1.000 & 1.000 & 1.000 & 1.000 & \multirow{12}{*}{0.358} \\
4 & pour liquid into cup & 1.000 & 0.000 & 1.000 & 0.667 & \\
9 & make toast & 0.000 & 1.000 & 0.000 & 0.333 & \\
13 & fold clothes & 0.200 & 0.000 & 0.000 & 0.067 & \\
20 & hang mugs & 0.000 & 0.000 & 0.000 & 0.000 & \\
24 & pack objects into box & 0.000 & 0.000 & 0.000 & 0.000 & \\
25 & store laptop and headphones & 0.000 & 0.000 & 0.000 & 0.000 & \\
27 & sweep blocks & 1.000 & 1.000 & 0.000 & 0.667 & \\
35 & push T & 0.000 & 0.000 & 0.000 & 0.000 & \\
37 & stack bowls & 1.000 & 1.000 & 0.000 & 0.667 & \\
38 & sort nesting dolls by size & 0.000 & 0.000 & 0.000 & 0.000 & \\
40 & arrange largest number & 0.300 & 0.400 & 1.000 & 0.567 & \\
\midrule
\multicolumn{7}{@{}l}{\textbf{Memory (6 tasks)}} \\
2 & cover blocks & 1.000 & 1.000 & 1.000 & 1.000 & \multirow{6}{*}{0.544} \\
12 & match and pick from conveyor & 1.000 & 0.000 & 1.000 & 0.667 & \\
28 & swap T & 0.000 & 0.000 & 0.000 & 0.000 & \\
32 & press by number & 1.000 & 1.000 & 0.000 & 0.667 & \\
39 & imitate sorting sequence & 0.500 & 0.300 & 1.000 & 0.600 & \\
42 & swap blocks & 1.000 & 0.000 & 0.000 & 0.333 & \\
\midrule
\multicolumn{7}{@{}l}{\textbf{Precision (8 tasks)}} \\
6 & pour balls into vase & 0.000 & 1.000 & 1.000 & 0.667 & \multirow{8}{*}{0.188} \\
7 & plug in charger & 0.000 & 0.000 & 0.000 & 0.000 & \\
16 & insert key & 0.150 & 0.000 & 0.150 & 0.100 & \\
17 & insert tubes & 0.000 & 0.000 & 0.000 & 0.000 & \\
21 & build tower & 0.000 & 0.300 & 0.300 & 0.200 & \\
26 & deposit coin & 0.000 & 0.000 & 0.200 & 0.067 & \\
29 & fasten screws & 0.200 & 1.000 & 0.200 & 0.467 & \\
41 & play Xylophone & 0.000 & 0.000 & 0.000 & 0.000 & \\
\midrule
\multicolumn{7}{@{}l}{\textbf{Long-Horizon (8 tasks)}} \\
10 & make kong & 0.000 & 0.000 & 0.000 & 0.000 & \multirow{8}{*}{0.175} \\
15 & classify objects & 0.000 & 0.150 & 0.000 & 0.050 & \\
18 & fill pen holder & 0.000 & 0.000 & 0.000 & 0.000 & \\
19 & fill egg holder & 0.000 & 0.000 & 0.000 & 0.000 & \\
22 & organize table & 0.250 & 0.000 & 0.250 & 0.167 & \\
23 & play stacking toy & 0.100 & 0.100 & 0.100 & 0.100 & \\
30 & put bottles into dustbin & 0.250 & 0.400 & 0.100 & 0.250 & \\
31 & play tic-tac-toe & 1.000 & 0.500 & 1.000 & 0.833 & \\
\midrule
\multicolumn{7}{@{}l}{\textbf{Open (8 tasks)}} \\
1 & align blocks & 1.000 & 0.000 & 0.000 & 0.333 & \multirow{8}{*}{0.317} \\
5 & pour by language & 0.000 & 0.000 & 0.000 & 0.000 & \\
8 & store tools in toolbox & 0.000 & 0.000 & 0.000 & 0.000 & \\
11 & pick from conveyor by image & 0.000 & 0.000 & 0.000 & 0.000 & \\
14 & classify objects by language & 0.000 & 0.100 & 0.100 & 0.067 & \\
33 & general pickup & 1.000 & 1.000 & 1.000 & 1.000 & \\
34 & solve equation & 1.000 & 1.000 & 0.000 & 0.667 & \\
36 & stack blocks by language & 1.000 & 0.200 & 0.200 & 0.467 & \\
\midrule
\multicolumn{6}{@{}l}{\textbf{Five-dimension mean}} & \textbf{0.3164} \\
\end{longtable}
\endgroup

\subsection{RoboCasa365}
\begingroup
\scriptsize
\setlength{\tabcolsep}{0pt}
\renewcommand{\arraystretch}{0.93}
\setlength{\LTcapwidth}{\suppdetailwidth}
\setlength{\LTleft}{84.5pt}
\setlength{\LTright}{84.5pt}
\begin{longtable}{@{}>{\raggedleft\arraybackslash}p{11pt}@{\hspace{3pt}}>{\raggedright\arraybackslash}p{105.5pt}*{2}{>{\raggedleft\arraybackslash}p{20pt}}@{\hspace{8pt}}>{\raggedleft\arraybackslash}p{11pt}@{\hspace{3pt}}>{\raggedright\arraybackslash}p{105.5pt}*{2}{>{\raggedleft\arraybackslash}p{20pt}}@{}}
\caption{Per-task RoboCasa365 scores for RoboCoder-27B. The table uses two side-by-side blocks; each block lists two evaluation instances, and no per-task mean is reported.}\label{tab:robocasa-details}\\
\toprule
ID & Task & I1 & I2 & ID & Task & I1 & I2 \\
   &      &      &      &    &      &      &      \\
\midrule
\endfirsthead
\toprule
ID & Task & I1 & I2 & ID & Task & I1 & I2 \\
   &      &      &      &    &      &      &      \\
\midrule
\endhead
\bottomrule
\endfoot
1 & CloseBlenderLid & 1.000 & 0.000 & 26 & RinseSinkBasin & 0.000 & 0.000 \\
2 & CloseFridge & 1.000 & 1.000 & 27 & ScrubCuttingBoard & 1.000 & 1.000 \\
3 & CloseToasterOvenDoor & 0.000 & 0.000 & 28 & SearingMeat & 0.000 & 0.000 \\
4 & CoffeeSetupMug & 0.000 & 1.000 & 29 & SetUpCuttingStation & 1.000 & 1.000 \\
5 & NavigateKitchen & 1.000 & 1.000 & 30 & StackBowlsCabinet & 1.000 & 0.000 \\
6 & OpenCabinet & 0.000 & 0.000 & 31 & SteamInMicrowave & 0.000 & 0.000 \\
7 & OpenDrawer & 1.000 & 1.000 & 32 & StirVegetables & 0.000 & 0.000 \\
8 & OpenStandMixerHead & 1.000 & 1.000 & 33 & StoreLeftoversInBowl & 1.000 & 0.000 \\
9 & PickPlaceCounterToCabinet & 1.000 & 0.000 & 34 & WashLettuce & 0.000 & 1.000 \\
10 & PickPlaceCounterToStove & 1.000 & 0.000 & 35 & ArrangeBreadBasket & 0.000 & 0.000 \\
11 & PickPlaceDrawerToCounter & 0.000 & 1.000 & 36 & ArrangeTea & 0.000 & 0.000 \\
12 & PickPlaceSinkToCounter & 1.000 & 1.000 & 37 & BreadSelection & 1.000 & 1.000 \\
13 & PickPlaceToasterToCounter & 1.000 & 1.000 & 38 & CategorizeCondiments & 0.000 & 1.000 \\
14 & SlideDishwasherRack & 1.000 & 1.000 & 39 & CuttingToolSelection & 1.000 & 0.000 \\
15 & TurnOffStove & 1.000 & 0.000 & 40 & GarnishPancake & 1.000 & 0.000 \\
16 & TurnOnElectricKettle & 1.000 & 1.000 & 41 & GatherTableware & 0.000 & 0.000 \\
17 & TurnOnMicrowave & 1.000 & 1.000 & 42 & HeatKebabSandwich & 0.000 & 0.000 \\
18 & TurnOnSinkFaucet & 1.000 & 1.000 & 43 & MakeIceLemonade & 0.000 & 1.000 \\
19 & DeliverStraw & 0.000 & 0.000 & 44 & PanTransfer & 1.000 & 1.000 \\
20 & GetToastedBread & 0.000 & 0.000 & 45 & PortionHotDogs & 0.000 & 0.000 \\
21 & KettleBoiling & 1.000 & 1.000 & 46 & RecycleBottlesByType & 0.000 & 1.000 \\
22 & LoadDishwasher & 1.000 & 1.000 & 47 & SeparateFreezerRack & 0.000 & 0.000 \\
23 & PackIdenticalLunches & 0.000 & 0.000 & 48 & WaffleReheat & 0.000 & 0.000 \\
24 & PreSoakPan & 0.000 & 0.000 & 49 & WashFruitColander & 0.000 & 0.000 \\
25 & PrepareCoffee & 0.000 & 0.000 & 50 & WeighIngredients & 0.000 & 0.000 \\
\midrule
\multicolumn{8}{@{}r@{}}{\textbf{Task 1--50 mean: 0.460}} \\
\end{longtable}
\endgroup

\subsection{Per-Object Outcomes}
The table below lists the runs in Fig.~\ref{fig:robot-hours}. A check is a success and a cross is a failure under the $5$\,cm rule of Sec.~\ref{sec:generalization}.

\setlength{\LTcapwidth}{\textwidth}
\begingroup
\scriptsize
\setlength{\tabcolsep}{1.6pt}
\renewcommand{\arraystretch}{0.92}
\setlength{\LTcapwidth}{\suppdetailwidth}
\setlength{\LTleft}{84.5pt}
\setlength{\LTright}{84.5pt}
\begin{longtable}{@{}r>{\raggedright\arraybackslash}p{2.35cm}*{13}{c}@{}}
\caption{Per-object outcome on the 100 objects. }
\label{tab:object-outcomes}\\
\toprule
\# & Object & 27B & Flash & luna & $\pi_{0.5}$ & F-WAM & G0.5 & OWAM & Ling & OVLA & ABot & Xiao & XVLA & Pick \\
\midrule
\endfirsthead
\toprule
\# & Object & 27B & Flash & luna & $\pi_{0.5}$ & F-WAM & G0.5 & OWAM & Ling & OVLA & ABot & Xiao & XVLA & Pick \\
\midrule
\endhead
\bottomrule
\endfoot
1 & alarm\_\allowbreak{}clock & $\checkmark$ & $\times$ & $\times$ & $\times$ & $\times$ & $\times$ & $\times$ & $\times$ & $\times$ & $\times$ & $\times$ & $\times$ & $\times$ \\
2 & apple & $\checkmark$ & $\checkmark$ & $\checkmark$ & $\checkmark$ & $\checkmark$ & $\checkmark$ & $\checkmark$ & $\checkmark$ & $\times$ & $\checkmark$ & $\times$ & $\times$ & $\times$ \\
3 & avocado & $\checkmark$ & $\checkmark$ & $\checkmark$ & $\checkmark$ & $\checkmark$ & $\checkmark$ & $\times$ & $\checkmark$ & $\times$ & $\times$ & $\checkmark$ & $\times$ & $\times$ \\
4 & banana & $\checkmark$ & $\checkmark$ & $\checkmark$ & $\times$ & $\checkmark$ & $\times$ & $\times$ & $\times$ & $\times$ & $\times$ & $\times$ & $\checkmark$ & $\times$ \\
5 & bar\_\allowbreak{}soap & $\checkmark$ & $\times$ & $\times$ & $\checkmark$ & $\checkmark$ & $\checkmark$ & $\checkmark$ & $\times$ & $\times$ & $\times$ & $\times$ & $\times$ & $\times$ \\
6 & baseball & $\times$ & $\checkmark$ & $\times$ & $\checkmark$ & $\checkmark$ & $\checkmark$ & $\checkmark$ & $\checkmark$ & $\checkmark$ & $\checkmark$ & $\checkmark$ & $\checkmark$ & $\times$ \\
7 & battery & $\checkmark$ & $\checkmark$ & $\checkmark$ & $\times$ & $\times$ & $\times$ & $\checkmark$ & $\checkmark$ & $\checkmark$ & $\times$ & $\times$ & $\times$ & $\times$ \\
8 & bell\_\allowbreak{}pepper & $\checkmark$ & $\checkmark$ & $\checkmark$ & $\checkmark$ & $\checkmark$ & $\checkmark$ & $\checkmark$ & $\times$ & $\checkmark$ & $\checkmark$ & $\checkmark$ & $\times$ & $\times$ \\
9 & blackboard\_\allowbreak{}eraser & $\checkmark$ & $\checkmark$ & $\checkmark$ & $\checkmark$ & $\checkmark$ & $\checkmark$ & $\checkmark$ & $\checkmark$ & $\times$ & $\checkmark$ & $\checkmark$ & $\times$ & $\times$ \\
10 & calculator & $\times$ & $\times$ & $\times$ & $\times$ & $\times$ & $\checkmark$ & $\times$ & $\times$ & $\times$ & $\times$ & $\times$ & $\times$ & $\times$ \\
11 & can\_\allowbreak{}of\_\allowbreak{}soda & $\checkmark$ & $\checkmark$ & $\checkmark$ & $\checkmark$ & $\checkmark$ & $\checkmark$ & $\checkmark$ & $\times$ & $\times$ & $\times$ & $\times$ & $\times$ & $\checkmark$ \\
12 & canned\_\allowbreak{}food & $\checkmark$ & $\checkmark$ & $\times$ & $\checkmark$ & $\checkmark$ & $\times$ & $\checkmark$ & $\checkmark$ & $\checkmark$ & $\times$ & $\checkmark$ & $\checkmark$ & $\times$ \\
13 & carrot & $\checkmark$ & $\checkmark$ & $\times$ & $\checkmark$ & $\checkmark$ & $\checkmark$ & $\checkmark$ & $\checkmark$ & $\checkmark$ & $\checkmark$ & $\checkmark$ & $\times$ & $\times$ \\
14 & cell\_\allowbreak{}phone & $\checkmark$ & $\checkmark$ & $\times$ & $\checkmark$ & $\times$ & $\times$ & $\times$ & $\times$ & $\times$ & $\times$ & $\times$ & $\times$ & $\checkmark$ \\
15 & cellulose\_\allowbreak{}tape\_\allowbreak{}dispenser & $\checkmark$ & $\checkmark$ & $\checkmark$ & $\checkmark$ & $\checkmark$ & $\checkmark$ & $\checkmark$ & $\times$ & $\times$ & $\times$ & $\times$ & $\times$ & $\times$ \\
16 & coffee\_\allowbreak{}cup & $\checkmark$ & $\times$ & $\times$ & $\checkmark$ & $\checkmark$ & $\checkmark$ & $\checkmark$ & $\checkmark$ & $\times$ & $\checkmark$ & $\checkmark$ & $\checkmark$ & $\times$ \\
17 & cucumber & $\checkmark$ & $\checkmark$ & $\checkmark$ & $\checkmark$ & $\checkmark$ & $\checkmark$ & $\checkmark$ & $\times$ & $\times$ & $\checkmark$ & $\checkmark$ & $\times$ & $\times$ \\
18 & deodorant\_\allowbreak{}stick & $\checkmark$ & $\checkmark$ & $\checkmark$ & $\checkmark$ & $\checkmark$ & $\checkmark$ & $\times$ & $\checkmark$ & $\times$ & $\times$ & $\times$ & $\times$ & $\checkmark$ \\
19 & dental\_\allowbreak{}floss & $\times$ & $\checkmark$ & $\checkmark$ & $\times$ & $\checkmark$ & $\times$ & $\checkmark$ & $\times$ & $\times$ & $\times$ & $\times$ & $\times$ & $\times$ \\
20 & digital\_\allowbreak{}camera & $\checkmark$ & $\checkmark$ & $\times$ & $\times$ & $\times$ & $\times$ & $\times$ & $\times$ & $\times$ & $\times$ & $\times$ & $\times$ & $\times$ \\
21 & duct\_\allowbreak{}tape & $\checkmark$ & $\checkmark$ & $\checkmark$ & $\times$ & $\times$ & $\times$ & $\times$ & $\times$ & $\times$ & $\times$ & $\times$ & $\times$ & $\checkmark$ \\
22 & easter\_\allowbreak{}egg & $\checkmark$ & $\checkmark$ & $\checkmark$ & $\checkmark$ & $\checkmark$ & $\times$ & $\checkmark$ & $\checkmark$ & $\checkmark$ & $\times$ & $\times$ & $\times$ & $\times$ \\
23 & egg & $\times$ & $\times$ & $\checkmark$ & $\checkmark$ & $\checkmark$ & $\checkmark$ & $\checkmark$ & $\checkmark$ & $\checkmark$ & $\checkmark$ & $\times$ & $\times$ & $\times$ \\
24 & eggplant & $\checkmark$ & $\checkmark$ & $\checkmark$ & $\checkmark$ & $\checkmark$ & $\checkmark$ & $\checkmark$ & $\times$ & $\times$ & $\times$ & $\checkmark$ & $\times$ & $\times$ \\
25 & garlic & $\checkmark$ & $\checkmark$ & $\checkmark$ & $\checkmark$ & $\checkmark$ & $\checkmark$ & $\checkmark$ & $\checkmark$ & $\checkmark$ & $\times$ & $\times$ & $\checkmark$ & $\times$ \\
26 & gaming\_\allowbreak{}controller & $\checkmark$ & $\checkmark$ & $\checkmark$ & $\times$ & $\times$ & $\times$ & $\times$ & $\times$ & $\times$ & $\times$ & $\times$ & $\times$ & $\times$ \\
27 & hockey\_\allowbreak{}puck & $\checkmark$ & $\checkmark$ & $\times$ & $\checkmark$ & $\checkmark$ & $\times$ & $\checkmark$ & $\checkmark$ & $\times$ & $\times$ & $\times$ & $\times$ & $\times$ \\
28 & inhaler & $\checkmark$ & $\checkmark$ & $\times$ & $\times$ & $\times$ & $\checkmark$ & $\checkmark$ & $\checkmark$ & $\checkmark$ & $\times$ & $\times$ & $\times$ & $\times$ \\
29 & lemon & $\checkmark$ & $\checkmark$ & $\checkmark$ & $\checkmark$ & $\checkmark$ & $\checkmark$ & $\checkmark$ & $\times$ & $\checkmark$ & $\checkmark$ & $\times$ & $\times$ & $\times$ \\
30 & lime & $\checkmark$ & $\checkmark$ & $\checkmark$ & $\checkmark$ & $\checkmark$ & $\checkmark$ & $\checkmark$ & $\checkmark$ & $\checkmark$ & $\checkmark$ & $\times$ & $\times$ & $\times$ \\
31 & lipstick & $\checkmark$ & $\checkmark$ & $\checkmark$ & $\checkmark$ & $\checkmark$ & $\checkmark$ & $\times$ & $\times$ & $\checkmark$ & $\times$ & $\times$ & $\times$ & $\times$ \\
32 & mouse & $\checkmark$ & $\checkmark$ & $\checkmark$ & $\checkmark$ & $\checkmark$ & $\checkmark$ & $\times$ & $\times$ & $\times$ & $\checkmark$ & $\times$ & $\times$ & $\times$ \\
33 & mug & $\checkmark$ & $\checkmark$ & $\checkmark$ & $\checkmark$ & $\checkmark$ & $\checkmark$ & $\checkmark$ & $\checkmark$ & $\times$ & $\times$ & $\times$ & $\times$ & $\checkmark$ \\
34 & orange & $\checkmark$ & $\checkmark$ & $\checkmark$ & $\checkmark$ & $\checkmark$ & $\times$ & $\checkmark$ & $\checkmark$ & $\checkmark$ & $\checkmark$ & $\times$ & $\times$ & $\times$ \\
35 & paper\_\allowbreak{}cup & $\checkmark$ & $\checkmark$ & $\times$ & $\checkmark$ & $\times$ & $\checkmark$ & $\checkmark$ & $\times$ & $\times$ & $\times$ & $\times$ & $\checkmark$ & $\times$ \\
36 & peach & $\checkmark$ & $\checkmark$ & $\checkmark$ & $\checkmark$ & $\times$ & $\times$ & $\times$ & $\checkmark$ & $\checkmark$ & $\checkmark$ & $\times$ & $\times$ & $\times$ \\
37 & pear & $\checkmark$ & $\checkmark$ & $\checkmark$ & $\checkmark$ & $\checkmark$ & $\checkmark$ & $\checkmark$ & $\times$ & $\times$ & $\checkmark$ & $\times$ & $\times$ & $\times$ \\
38 & pepper\_\allowbreak{}grinder & $\checkmark$ & $\checkmark$ & $\checkmark$ & $\checkmark$ & $\times$ & $\checkmark$ & $\checkmark$ & $\checkmark$ & $\times$ & $\times$ & $\checkmark$ & $\times$ & $\times$ \\
39 & pill\_\allowbreak{}bottle & $\checkmark$ & $\checkmark$ & $\times$ & $\checkmark$ & $\checkmark$ & $\times$ & $\checkmark$ & $\checkmark$ & $\checkmark$ & $\times$ & $\times$ & $\times$ & $\times$ \\
40 & pillar\_\allowbreak{}candle & $\checkmark$ & $\checkmark$ & $\times$ & $\times$ & $\times$ & $\times$ & $\times$ & $\times$ & $\times$ & $\times$ & $\times$ & $\times$ & $\times$ \\
41 & plum & $\checkmark$ & $\checkmark$ & $\checkmark$ & $\checkmark$ & $\checkmark$ & $\times$ & $\checkmark$ & $\checkmark$ & $\checkmark$ & $\times$ & $\times$ & $\times$ & $\times$ \\
42 & pool\_\allowbreak{}ball & $\checkmark$ & $\checkmark$ & $\checkmark$ & $\checkmark$ & $\checkmark$ & $\checkmark$ & $\checkmark$ & $\checkmark$ & $\checkmark$ & $\checkmark$ & $\times$ & $\times$ & $\times$ \\
43 & portable\_\allowbreak{}game\_\allowbreak{}console & $\times$ & $\checkmark$ & $\times$ & $\times$ & $\times$ & $\checkmark$ & $\times$ & $\checkmark$ & $\times$ & $\times$ & $\times$ & $\times$ & $\times$ \\
44 & potato & $\checkmark$ & $\checkmark$ & $\checkmark$ & $\checkmark$ & $\checkmark$ & $\checkmark$ & $\checkmark$ & $\times$ & $\checkmark$ & $\checkmark$ & $\times$ & $\checkmark$ & $\times$ \\
45 & salt\_\allowbreak{}shaker & $\checkmark$ & $\checkmark$ & $\checkmark$ & $\checkmark$ & $\checkmark$ & $\checkmark$ & $\checkmark$ & $\checkmark$ & $\checkmark$ & $\times$ & $\times$ & $\times$ & $\times$ \\
46 & soda\_\allowbreak{}cup & $\times$ & $\checkmark$ & $\checkmark$ & $\checkmark$ & $\times$ & $\times$ & $\times$ & $\times$ & $\times$ & $\checkmark$ & $\times$ & $\checkmark$ & $\checkmark$ \\
47 & stapler & $\checkmark$ & $\checkmark$ & $\checkmark$ & $\checkmark$ & $\checkmark$ & $\checkmark$ & $\times$ & $\times$ & $\times$ & $\times$ & $\checkmark$ & $\times$ & $\times$ \\
48 & tennis\_\allowbreak{}ball & $\checkmark$ & $\times$ & $\checkmark$ & $\checkmark$ & $\checkmark$ & $\checkmark$ & $\checkmark$ & $\checkmark$ & $\checkmark$ & $\checkmark$ & $\times$ & $\times$ & $\times$ \\
49 & toy\_\allowbreak{}dice & $\checkmark$ & $\checkmark$ & $\checkmark$ & $\times$ & $\checkmark$ & $\times$ & $\checkmark$ & $\checkmark$ & $\checkmark$ & $\times$ & $\times$ & $\times$ & $\times$ \\
50 & wallet & $\checkmark$ & $\checkmark$ & $\times$ & $\checkmark$ & $\times$ & $\checkmark$ & $\times$ & $\times$ & $\times$ & $\times$ & $\times$ & $\times$ & $\times$ \\
51 & artichoke & $\checkmark$ & $\checkmark$ & $\checkmark$ & $\checkmark$ & $\checkmark$ & $\times$ & $\checkmark$ & $\checkmark$ & $\checkmark$ & $\times$ & $\times$ & $\times$ & $\times$ \\
52 & baby\_\allowbreak{}bottle & $\checkmark$ & $\checkmark$ & $\checkmark$ & $\checkmark$ & $\checkmark$ & $\checkmark$ & $\checkmark$ & $\checkmark$ & $\times$ & $\times$ & $\times$ & $\times$ & $\times$ \\
53 & beaker & $\checkmark$ & $\checkmark$ & $\checkmark$ & $\checkmark$ & $\times$ & $\times$ & $\times$ & $\checkmark$ & $\times$ & $\times$ & $\times$ & $\checkmark$ & $\times$ \\
54 & beefsteak\_\allowbreak{}tomato & $\checkmark$ & $\checkmark$ & $\checkmark$ & $\times$ & $\checkmark$ & $\checkmark$ & $\checkmark$ & $\times$ & $\times$ & $\checkmark$ & $\times$ & $\checkmark$ & $\times$ \\
55 & beer\_\allowbreak{}glass & $\checkmark$ & $\checkmark$ & $\checkmark$ & $\checkmark$ & $\times$ & $\times$ & $\checkmark$ & $\times$ & $\times$ & $\times$ & $\checkmark$ & $\times$ & $\times$ \\
56 & beet & $\checkmark$ & $\checkmark$ & $\times$ & $\times$ & $\times$ & $\times$ & $\times$ & $\times$ & $\times$ & $\times$ & $\times$ & $\times$ & $\times$ \\
57 & bottle\_\allowbreak{}of\_\allowbreak{}hot\_\allowbreak{}sauce & $\checkmark$ & $\checkmark$ & $\times$ & $\checkmark$ & $\times$ & $\checkmark$ & $\times$ & $\checkmark$ & $\times$ & $\times$ & $\checkmark$ & $\times$ & $\times$ \\
58 & bowl & $\checkmark$ & $\checkmark$ & $\times$ & $\checkmark$ & $\times$ & $\times$ & $\times$ & $\times$ & $\times$ & $\times$ & $\times$ & $\times$ & $\checkmark$ \\
59 & brussels\_\allowbreak{}sprouts & $\checkmark$ & $\checkmark$ & $\times$ & $\checkmark$ & $\checkmark$ & $\checkmark$ & $\checkmark$ & $\times$ & $\checkmark$ & $\times$ & $\times$ & $\times$ & $\times$ \\
60 & can\_\allowbreak{}of\_\allowbreak{}beans & $\checkmark$ & $\checkmark$ & $\times$ & $\checkmark$ & $\times$ & $\times$ & $\times$ & $\times$ & $\times$ & $\times$ & $\times$ & $\checkmark$ & $\times$ \\
61 & can\_\allowbreak{}of\_\allowbreak{}sardines & $\checkmark$ & $\checkmark$ & $\times$ & $\checkmark$ & $\times$ & $\checkmark$ & $\checkmark$ & $\times$ & $\times$ & $\times$ & $\times$ & $\checkmark$ & $\times$ \\
62 & canteen & $\checkmark$ & $\checkmark$ & $\checkmark$ & $\times$ & $\times$ & $\times$ & $\times$ & $\times$ & $\times$ & $\times$ & $\times$ & $\times$ & $\times$ \\
63 & carton\_\allowbreak{}of\_\allowbreak{}milk & $\checkmark$ & $\checkmark$ & $\checkmark$ & $\checkmark$ & $\checkmark$ & $\times$ & $\times$ & $\checkmark$ & $\times$ & $\times$ & $\times$ & $\times$ & $\times$ \\
64 & fried\_\allowbreak{}chicken & $\checkmark$ & $\checkmark$ & $\checkmark$ & $\times$ & $\times$ & $\times$ & $\times$ & $\times$ & $\times$ & $\times$ & $\times$ & $\times$ & $\times$ \\
65 & cinnamon\_\allowbreak{}roll & $\checkmark$ & $\checkmark$ & $\checkmark$ & $\checkmark$ & $\checkmark$ & $\checkmark$ & $\checkmark$ & $\checkmark$ & $\checkmark$ & $\times$ & $\times$ & $\times$ & $\times$ \\
66 & marker & $\times$ & $\checkmark$ & $\times$ & $\checkmark$ & $\checkmark$ & $\checkmark$ & $\checkmark$ & $\times$ & $\checkmark$ & $\times$ & $\times$ & $\times$ & $\times$ \\
67 & conch & $\checkmark$ & $\checkmark$ & $\checkmark$ & $\times$ & $\times$ & $\times$ & $\times$ & $\times$ & $\times$ & $\times$ & $\times$ & $\times$ & $\times$ \\
68 & corkscrew & $\checkmark$ & $\checkmark$ & $\checkmark$ & $\times$ & $\checkmark$ & $\checkmark$ & $\times$ & $\checkmark$ & $\times$ & $\times$ & $\checkmark$ & $\times$ & $\times$ \\
69 & gravy\_\allowbreak{}boat & $\checkmark$ & $\checkmark$ & $\checkmark$ & $\times$ & $\times$ & $\times$ & $\times$ & $\times$ & $\times$ & $\times$ & $\times$ & $\times$ & $\times$ \\
70 & eyeglasses & $\checkmark$ & $\times$ & $\times$ & $\times$ & $\times$ & $\times$ & $\times$ & $\times$ & $\times$ & $\times$ & $\times$ & $\times$ & $\checkmark$ \\
71 & french\_\allowbreak{}press & $\checkmark$ & $\times$ & $\checkmark$ & $\times$ & $\times$ & $\times$ & $\times$ & $\times$ & $\times$ & $\times$ & $\times$ & $\times$ & $\times$ \\
72 & gift\_\allowbreak{}box & $\times$ & $\checkmark$ & $\checkmark$ & $\times$ & $\times$ & $\times$ & $\times$ & $\times$ & $\times$ & $\times$ & $\times$ & $\times$ & $\checkmark$ \\
73 & glue\_\allowbreak{}stick & $\times$ & $\checkmark$ & $\checkmark$ & $\times$ & $\checkmark$ & $\checkmark$ & $\checkmark$ & $\checkmark$ & $\checkmark$ & $\checkmark$ & $\checkmark$ & $\times$ & $\times$ \\
74 & grapefruit & $\checkmark$ & $\checkmark$ & $\checkmark$ & $\times$ & $\times$ & $\times$ & $\times$ & $\times$ & $\times$ & $\times$ & $\times$ & $\times$ & $\times$ \\
75 & hard\_\allowbreak{}drive & $\checkmark$ & $\checkmark$ & $\checkmark$ & $\checkmark$ & $\checkmark$ & $\checkmark$ & $\times$ & $\times$ & $\times$ & $\times$ & $\times$ & $\times$ & $\times$ \\
76 & hotdog & $\checkmark$ & $\checkmark$ & $\checkmark$ & $\checkmark$ & $\times$ & $\checkmark$ & $\times$ & $\checkmark$ & $\times$ & $\times$ & $\times$ & $\times$ & $\times$ \\
77 & spray\_\allowbreak{}can & $\checkmark$ & $\times$ & $\checkmark$ & $\checkmark$ & $\times$ & $\checkmark$ & $\times$ & $\times$ & $\times$ & $\times$ & $\times$ & $\times$ & $\times$ \\
78 & jade\_\allowbreak{}roller & $\times$ & $\times$ & $\checkmark$ & $\times$ & $\checkmark$ & $\checkmark$ & $\checkmark$ & $\checkmark$ & $\times$ & $\times$ & $\times$ & $\times$ & $\times$ \\
79 & jar\_\allowbreak{}of\_\allowbreak{}jam & $\checkmark$ & $\checkmark$ & $\checkmark$ & $\checkmark$ & $\checkmark$ & $\times$ & $\times$ & $\checkmark$ & $\times$ & $\checkmark$ & $\checkmark$ & $\checkmark$ & $\times$ \\
80 & jigger & $\checkmark$ & $\checkmark$ & $\checkmark$ & $\checkmark$ & $\checkmark$ & $\times$ & $\times$ & $\checkmark$ & $\times$ & $\times$ & $\times$ & $\times$ & $\times$ \\
81 & light\_\allowbreak{}bulb & $\checkmark$ & $\checkmark$ & $\checkmark$ & $\times$ & $\checkmark$ & $\checkmark$ & $\checkmark$ & $\checkmark$ & $\times$ & $\times$ & $\times$ & $\times$ & $\times$ \\
82 & lighter & $\checkmark$ & $\checkmark$ & $\checkmark$ & $\times$ & $\times$ & $\checkmark$ & $\times$ & $\times$ & $\times$ & $\times$ & $\times$ & $\times$ & $\times$ \\
83 & mango & $\checkmark$ & $\checkmark$ & $\checkmark$ & $\checkmark$ & $\checkmark$ & $\checkmark$ & $\checkmark$ & $\times$ & $\checkmark$ & $\checkmark$ & $\times$ & $\times$ & $\times$ \\
84 & match\_\allowbreak{}box & $\checkmark$ & $\checkmark$ & $\times$ & $\checkmark$ & $\checkmark$ & $\checkmark$ & $\checkmark$ & $\checkmark$ & $\checkmark$ & $\checkmark$ & $\checkmark$ & $\times$ & $\times$ \\
85 & mousetrap & $\checkmark$ & $\checkmark$ & $\checkmark$ & $\checkmark$ & $\checkmark$ & $\checkmark$ & $\times$ & $\checkmark$ & $\checkmark$ & $\times$ & $\times$ & $\times$ & $\times$ \\
86 & pencil\_\allowbreak{}case & $\checkmark$ & $\checkmark$ & $\checkmark$ & $\checkmark$ & $\times$ & $\checkmark$ & $\times$ & $\times$ & $\times$ & $\times$ & $\times$ & $\times$ & $\times$ \\
87 & plier & $\times$ & $\checkmark$ & $\times$ & $\times$ & $\checkmark$ & $\checkmark$ & $\checkmark$ & $\times$ & $\times$ & $\times$ & $\times$ & $\times$ & $\times$ \\
88 & pomegranate & $\checkmark$ & $\times$ & $\checkmark$ & $\checkmark$ & $\checkmark$ & $\checkmark$ & $\checkmark$ & $\checkmark$ & $\times$ & $\checkmark$ & $\times$ & $\checkmark$ & $\times$ \\
89 & razor & $\times$ & $\times$ & $\times$ & $\times$ & $\checkmark$ & $\checkmark$ & $\times$ & $\checkmark$ & $\times$ & $\times$ & $\times$ & $\times$ & $\times$ \\
90 & saucepan & $\checkmark$ & $\times$ & $\times$ & $\times$ & $\times$ & $\times$ & $\times$ & $\times$ & $\times$ & $\times$ & $\times$ & $\times$ & $\times$ \\
91 & screwdriver & $\checkmark$ & $\checkmark$ & $\checkmark$ & $\checkmark$ & $\checkmark$ & $\checkmark$ & $\checkmark$ & $\times$ & $\checkmark$ & $\times$ & $\checkmark$ & $\times$ & $\times$ \\
92 & scrub\_\allowbreak{}brush & $\checkmark$ & $\checkmark$ & $\checkmark$ & $\checkmark$ & $\checkmark$ & $\checkmark$ & $\checkmark$ & $\times$ & $\checkmark$ & $\checkmark$ & $\checkmark$ & $\times$ & $\times$ \\
93 & snow\_\allowbreak{}globe & $\checkmark$ & $\checkmark$ & $\times$ & $\checkmark$ & $\times$ & $\times$ & $\times$ & $\checkmark$ & $\times$ & $\checkmark$ & $\times$ & $\times$ & $\times$ \\
94 & teacup & $\times$ & $\checkmark$ & $\times$ & $\checkmark$ & $\times$ & $\times$ & $\checkmark$ & $\times$ & $\times$ & $\checkmark$ & $\checkmark$ & $\checkmark$ & $\times$ \\
95 & tube\_\allowbreak{}of\_\allowbreak{}toothpaste & $\checkmark$ & $\checkmark$ & $\checkmark$ & $\times$ & $\checkmark$ & $\checkmark$ & $\checkmark$ & $\checkmark$ & $\checkmark$ & $\checkmark$ & $\checkmark$ & $\times$ & $\checkmark$ \\
96 & vegetable\_\allowbreak{}peeler & $\checkmark$ & $\checkmark$ & $\checkmark$ & $\checkmark$ & $\times$ & $\checkmark$ & $\checkmark$ & $\times$ & $\checkmark$ & $\checkmark$ & $\checkmark$ & $\times$ & $\times$ \\
97 & water\_\allowbreak{}glass & $\times$ & $\times$ & $\checkmark$ & $\checkmark$ & $\checkmark$ & $\times$ & $\times$ & $\checkmark$ & $\times$ & $\times$ & $\times$ & $\times$ & $\times$ \\
98 & toy\_\allowbreak{}figure & $\checkmark$ & $\checkmark$ & $\checkmark$ & $\times$ & $\times$ & $\times$ & $\times$ & $\times$ & $\times$ & $\times$ & $\times$ & $\times$ & $\times$ \\
99 & spirit\_\allowbreak{}lamp & $\checkmark$ & $\checkmark$ & $\times$ & $\checkmark$ & $\checkmark$ & $\times$ & $\checkmark$ & $\checkmark$ & $\checkmark$ & $\times$ & $\times$ & $\checkmark$ & $\times$ \\
100 & zucchini & $\checkmark$ & $\checkmark$ & $\times$ & $\checkmark$ & $\checkmark$ & $\checkmark$ & $\checkmark$ & $\times$ & $\times$ & $\checkmark$ & $\times$ & $\times$ & $\times$ \\
\end{longtable}
\endgroup

\twocolumn

\section{Code Audit Behind Table~\ref{tab:compare}}
\label{sec:supp-audit}

\paragraph{OpenRUA (commit 1e460cd).}
Stock Claude Code or Codex runs in a container with a persistent workspace and skill files.
Tools are plain Python files (writable in execution modes; robot config read-only in project mode); camera tools save images/arrays on disk for on-demand reading with numpy, scipy, cv2 and PyKDL.
The simulator exposes joint trajectories, twists and gripper; no training code is released.

\paragraph{RAX (\texttt{robodojo-behavior/native}, 16e75db).}
Unmodified Codex/Claude loops maintain persistent code and skills, run Python over saved RGB/depth/state, and set joint targets in simulation (TCP poses on Franka).
Tools sit behind an unreadable host-side MCP backend; the model is frozen.

\paragraph{HexaShell (e0190dd).}
Uses a custom chat-gateway loop with \texttt{python} and \texttt{finish}; robot tools are host-side RPC calls in a Landlock read-only package.
The lowest visible command is end-effector goto (joints blocked by whitelist), perception is pulled via image/look calls, files persist only within an episode, and the API varies per setting.

\paragraph{Auriga (0fd2337).}
A pinned Codex CLI sees only read-only tool snapshots, a served URDF/briefing, Cartesian end-effector waypoints, and a program-run tool over saved artifacts.
An experiment queue iterates the harness only; no model is trained.

\paragraph{uni-robot-agent (d5e8085).}
Claude Code or Codex runs via an SDK with a pluggable registry, reads URDF/manuals, and runs Bash/Python over saved images and state.
Skills are opt-in Markdown; served tools take end-effector poses plus gripper and auto-attach images. Only a Codex trajectory exporter is included (no trainer).

\paragraph{PyRUA-Lean (e754cf4).}
Hard-coded Codex CLI, restricted Claude Code and a custom OpenAI loop run per-episode wiped workspaces whose primitives live in sandbox-protected library code outside the workspace.
Control uses end-effector move, servo and skills; the 4B coder is reported without released training code.

\paragraph{ASPIRE (NVlabs repository).}
Stock Claude Code or Codex runs via skills and MCP with no model--harness loop; a Markdown skill library stores reusable snippets, step takes per-tick joint targets, and observation is pulled.
Prompts gate edits and exclude URDF reading; no weights are trained. Only the released repository was audited for Table~\ref{tab:compare}.

\paragraph{TinyAstral (8b21b1c).}
The acting agent has no shell, file access or code execution, only fixed semantic tool calls with pushed camera images.
An external master agent edits the harness offline in a human-gated loop; editing and SFT run sequentially without released export scripts.

\paragraph{SenseNova-RoboRSI (e17ddee).}
Supports pinned Codex CLI, Claude Agent SDK, or its default tool loop; tools are served over MCP/HTTP or pinned in containers, exposing joint targets (one benchmark) and OSC deltas (another).
The repository has no training or harness-editing loop (deployment script is a placeholder); Table~\ref{tab:compare} reflects the reported joint training, whose initialization and schedule remain undisclosed.

\paragraph{ReKep and CoPa.}
Both are fixed VLM pipelines outputting constraint cost functions (ReKep) or five constraint phrasings (CoPa) solved by fixed code, with no agent loop, general code execution, or training.

\paragraph{Additional related systems.}
RHO~\citep{elmaaroufi2026rho}, Guava~\citep{liu2026guava}, and AgenticRoboRL~\citep{liu2026agenticroborl} were also inspected and fall outside the table's scope (fixed tool loop for Guava, training wrapper for AgenticRoboRL, unreleased benchmark code for RHO).

\paragraph{Audit scope.}
The audit is based on static code inspection of a single released commit per repository rather than runtime execution, so behavior in other branches or unreleased paper versions may differ; other branches of RAX/PyRUA-Lean, HexaShell's external analysis agents, direct environment access in URA, and paper-only ASPIRE details lie outside this static snapshot.

\fi

\clearpage
{
    \small
    \bibliographystyle{ieeenat_fullname}
    \bibliography{main}
}

\end{document}